\documentclass[fleqn]{article}
\usepackage[english]{babel}
\usepackage{fullpage}

\usepackage{amsmath,amssymb,amsfonts,amsthm}
\usepackage[fleqn,tbtags]{mathtools}
\usepackage{bbm}
\usepackage{fancyhdr}
\usepackage{graphicx}
\usepackage{epsfig}
\usepackage{color}
\usepackage{xspace}
\usepackage{url}
\usepackage [autostyle, english = american]{csquotes}
\usepackage{ifpdf}

\ifpdf    
\usepackage{hyperref}
\else    
\usepackage[hypertex]{hyperref}
\fi

\MakeOuterQuote{"}

\newtheorem{theorem}{Theorem}
\newtheorem{lemma}[theorem]{Lemma}
\newtheorem{corollary}[theorem]{Corollary}
\newtheorem{definition}{Definition}

\newtheorem{claim}[theorem]{Claim}

\newcommand {\ignore} [1] {}
\def \eqdef {:=}

\DeclareMathOperator{\supp}{supp}

\providecommand{\eqdef}{:=}

\newcommand{\etal}{{\em et al.\ }\xspace}

\newcommand{\R}{\mathbb{R}}
\newcommand{\D}{\mathcal{D}}
\newcommand{\A}{\mathcal{A}}

\newcommand{\E}{\mathbb{E}}
\newcommand{\eps}{\varepsilon}

\newcommand{\Xs}{\mathcal{X}}

\newcommand{\y}{\hat{y}}

\newcommand{\ellsp}{\hat{\ell}}

\newcommand{\Loss}{\mathcal{L}}
\newcommand{\Dy}{\D_{\y}}
\newcommand{\dotp}[2]{\left\langle #1 , #2 \right\rangle}
\DeclareMathOperator*{\sign}{sign}

\title{Near-Tight Margin-Based Generalization Bounds for Support Vector Machines}

\author{Allan Gr{\o}nlund \thanks{Computer Science Department. Aarhus University. \texttt{jallan@cs.au.dk}.} \qquad 
Lior Kamma \thanks{Computer Science Department. Aarhus University. Supported by a Villum Young Investigator Grant \texttt{lior.kamma@cs.au.dk}.} 
\qquad 
Kasper Green Larsen \thanks{Computer Science Department. Aarhus University. Supported by a Villum Young
    Investigator Grant, an AUFF Starting Grant and a DFF Sapere Aude Starting Grant. \texttt{larsen@cs.au.dk}. }
}

\date{}

\begin{document}

\maketitle

\begin{abstract}
Support Vector Machines (SVMs) are among the most fundamental tools for binary classification. 
In its simplest formulation, an SVM produces a hyperplane separating two classes of data 
using the largest possible margin to the data. 
The focus on maximizing the margin has been well motivated through numerous generalization bounds. 
In this paper, we revisit and improve the classic generalization bounds in terms of margins. 
Furthermore, we complement our new generalization bound by a nearly matching lower bound, 
thus almost settling the generalization performance of SVMs in terms of margins.
\end{abstract}

\section{Introduction}
Since their introduction~\cite{Vapnik:1982, Cortes1995} 
{\em Support Vector Machines (SVMs)} have continued to be among the most popular classification algorithms. 
In the most basic setup an SVM produces, upon receiving a training data set, a classifier by finding a maximum margin hyperplane separating the data.
More formally, given a training data set $S = \{x_1,\dots,x_m\}$ of $m$ samples in $\R^d$, 
each with a label $y_i \in \{-1,+1\}$, an SVM finds a unit vector $w \in \R^d$ such that 
$y_i \langle x_i, w\rangle \geq \theta$ for all $i$, with the largest possible value of the margin $\theta$. Note that one often includes a bias parameter $b$ such that one instead requires $y_i (\langle x_i, w \rangle+b) \geq \theta$. As $b$ has no relevance on this work we ignore it for notational simplicity.
The predicted label on a new data data point $x \in \R^d$, is simply $\sign(\langle x, w\rangle)$. 
When the data is linearly separable, that is there exists a vector $w$ with $y_i \langle x_i, w\rangle > 0$ 
for all $i$, then the maximum margin hyperplane $w$ is the solution to the 
following convex optimization problem, which is often referred to as the {\em hard margin SVM}.
\begin{equation}
\begin{split}
&\min_w \; \|w\|_2^2\\
&s.t. \; y_i \langle x_i, w\rangle \geq 1 \;\; \forall i.
\end{split}
\label{eq:hardMargin}
\end{equation}

Note that the maximum margin hyperplane is not necessarily a vector $w$ of unit norm. If we however let $w^* = w/\|w\|_2$, then by linearity, we get a unit vector $w^*$ such that $y_i \langle x_i ,w^* \rangle \geq 1/\|w\|_2$ for all $i$.
That is, the margin becomes at least $1/\|w\|_2$ for all $(x_i,y_i)$.

As data is typically not linearly separable, one often considers a relaxed variant of the above optimization problem, 
known as {\em soft margin SVM} \cite{Cortes1995}. 

\begin{equation}
\begin{split}
\min_{w, \xi} \; &\|w\|_2^2 + \lambda \sum\nolimits_i \xi_i \\
s.t. \; &y_i \langle x_i, w\rangle \geq 1 -\xi_i \;\; \forall i.\\
&\xi_i \ge 0 \;\; \forall i.
\end{split}
\label{eq:softMargin}
\end{equation}

Here $\lambda \geq 0$ is a hyper parameter which, roughly speaking, controls the tradeoff between the magnitude 
of the margin $\theta = 1/\|w\|_2$ and the number of data points with margin significantly less than $\theta$.
The soft margin optimization problem is also convex and can be solved efficiently. 

A key reason for the success of SVMs is the extensive study and ubiquitousness of kernels (see e.g. \cite{Boser:1992}). By allowing efficient calculation of inner products in high (or even infinite) dimensional spaces, kernels make it possible to apply  SVMs in these spaces through feature transforms without actually having to compute the feature transform, neither during training or prediction.
Predictions are efficient since they only need to consider the support vectors. These are the sample data points $(x,y)$ that are not {\em strictly} on the correct side of the margin of the hyperplane, meaning that $y\dotp{x}{w} \le \theta$.

Feature transforms, like the application of a kernel, often drastically increase the dimensionality of the input domain, directly increasing the  the VC-dimension of the hypothesis set (the set of hyperplanes) the same way.
Thus one might worry about overfitting. However, SVMs, even with the Gaussian kernel that maps to an infinite dimensional space, often generalize well to new data points in practice.
Explaining this phenomenon has been the focus of much theoretical work, see e.g.~\cite{Vapnik:1982, Bartlett98generalizationperformance, DBLP:journals/jmlr/BartlettM02},
with probably the most prominent and simplest explanations being based on generalization bounds involving margins.
These margin generalization bounds show that, as long as a hypothesis vector  has large margins on most training data, then the hypothesis generalizes well to new data, independent of the dimension of the data.
Further strengthening these generalization bounds and our understanding of the influence of margins is the focus of this paper. We start by reviewing some of the previous margin-based generalization bounds for SVMs.

\subsection{Previous Generalization Bounds}
In what follows we review previous generalization bounds for SVMs. We have focused on the most classic bounds, taking only the margin $\theta$, the radius $R$ of the input space, and the number of data samples $m$ into account.
We have rephrased the previous theorems to put them all into the same form, allowing for easier comparison between them.
Throughout $X$ denotes the input space, $\D$ a distribution  over $X \times \{-1,1\}$, and $\Loss_\D(w)$ the out-of-sample error for a vector $w$. That is $\Loss_\D(w) = \Pr_{(x,y) \sim {\cal D}}\left[ \sign(\dotp{x}{w}) \neq y \right] = \Pr_{(x,y) \sim {\cal D}}\left[ y\dotp{x}{w} \le 0 \right]$.
Given a training set $S$ and a margin $\theta$, $\Loss^\theta_S(w)$ denotes the in-sample margin error for a vector $w$, i.e.  $\Loss^\theta_S(w) = \Pr_{(x,y)\sim S}\left[y\dotp{x}{w} \le \theta \right]$, where $(x,y)\sim S$ means that $(x,y)$ is sampled from $S$ uniformly at random.

The first work trying to explain the generalization performance of SVMs through margins is due to Bartlett and Shawe-Taylor~\cite{Bartlett98generalizationperformance}. They first consider the linearly separable case/hard margin SVM and prove the following generalization if all samples have margins at least $\theta$:
\begin{theorem}
\label{thm:bartlettNoErrors}[Bartlett and Shawe-Taylor~\cite{Bartlett98generalizationperformance}]
Let $d \in \mathbb{N}^+$ and let $R > 0$. Denote by $X$ the ball of radius $R$ in $\mathbb{R}^d$ and let ${\cal D}$ be any distribution over $X \times \{-1,1\}$. For every $\delta > 0$, it holds with probability at least $1-\delta$ over a set of $m$ samples $S \sim {\cal D}^m$, that for every $w \in \mathbb{R}^d$ with $\|w\|_2 \leq 1$, if all samples $(x,y) \in S$ have margin (i.e. $y\langle x, w \rangle$) at least $\theta > 0$, then:
$$\Loss_\D(w)\le O\left(\frac{(R/\theta)^2\ln^2 m + \ln(1/\delta)}{m} \right).$$
\end{theorem}
They complemented their bound with a generalization bound for the soft margin SVM setting, showing that in addition for all $\theta >0$,
\begin{equation*}
\Loss_\D(w)\le \Loss^\theta_S(w) +O\left(\sqrt{\frac{(R/\theta)^2\ln^2 m + \ln(1/\delta)}{m} }\right) \;.
\end{equation*} 

Notice how the generalization error in the soft margin case is larger as $\sqrt{x} \ge x$ for $x \in [0,1]$. This fits well with classic VC-dimension generalization bounds for the realizable and non-realizable setting, see e.g.~\cite{VC71, EHKV89, AB09}.

This bound was later improved by Bartlett and Mendelson~\cite{DBLP:journals/jmlr/BartlettM02}, who showed, using Rademacher complexity, that for all $\theta >0$, 
\begin{equation}
\Loss_\D(w) \le \Loss_S^\theta(w) + O\left(\sqrt{\frac{(R/\theta)^2+\ln(1/\delta)}{m}}\right)\;.
\label{eq:bartlettMendelsonWithErrors}
\end{equation}
Ignoring logarithmic factors and the dependency on $\delta$, both bounds show similar dependencies on the radius of the point set $R$, the margin $\theta$ and the number of samples $m$. The dependency on $R/\theta$ also fits well with the intuition that scaling the data distribution should not change the generalization performance. Finally notice how the soft margin bounds allow one to consider any margin $\theta$, not just the smallest over all samples, and then pay an additive term proportional to the fraction of points in the sample with margin less than $\theta$ (i.e. $\Loss_S^\theta(w) = \Pr_{(x,y)\in_R S}[y \dotp{x}{w} \le \theta ]$).

Finally, the work by McAllester~\cite{DBLP:conf/colt/McAllester03}, uses a PAC-Bayes argument to give a bound that attempts to interpolate between the hard margin and soft margin case. His bound shows that for all $\theta>0$, we have:
\begin{equation}
  \label{eq:mcallester}
 \Loss_\D(w) \le \Loss_S^\theta(w) + O\left(\frac{(R/\theta)^2\ln m }{m} +  \sqrt{\frac{(R/\theta)^2\ln m}{m} \cdot \Loss_S^\theta(w)}\right) + O\left(\sqrt{\frac{\ln m + \ln(1/\delta)}{m} }\right)\;.
\end{equation}
Notice that in the hard margin case, we have $\Loss_S^\theta(w) = 0$ and thus the above simplifies to $O((R/\theta)^2 \ln (m)/m) + O(\sqrt{(\ln m + \ln(1/\delta))/m})$. The first term is an $\ln m$ factor better than the hard margin bound by Bartlett and Shawe-Taylor (Theorem~\ref{thm:bartlettNoErrors}), but unfortunately it is dominated by the $\sqrt{(\ln m + \ln(1/\delta))/ m}$ term for all but very small margins ($\theta$ must be less than  $R(\ln(m)/m)^{1/4}$).

These classic bounds have not seen any improvements for almost two decades, even though we have no generalization lower bounds that rule out further improvements.
Generalization bounds for SVMs that are independent of the dimensionality of the space has also been proved based on the (expected) number of support vectors \cite{Vapnik:1982}.

\subsection{Our Contributions}
Our first main contribution is an improvement over the known margin-based generalization bounds for a large range of parameters. Our new generalization bound is as follows:

\begin{theorem}
\label{thm:mainupper}
Let $d \in \mathbb{N}^+$ and let $R > 0$. Denote by $X$ the ball of radius $R$ in $\mathbb{R}^d$ and let ${\cal D}$ be any distribution over $X \times \{-1,1\}$. For every $\delta > 0$, it holds with probability at least $1-\delta$ over a set of $m$ samples $S \sim {\cal D}^m$, that for every $w \in \mathbb{R}^d$ with $\|w\|_2 \leq 1$ and every margin $\theta > 0$, we have

$$\Loss_\D(w) \le \Loss_S^\theta(w) + O\left(\frac{(R/\theta)^2\ln m + \ln(1/\delta)}{m} + \sqrt{\frac{(R/\theta)^2\ln m + \ln(1/\delta)}{m} \cdot \Loss_S^\theta(w)}\right) \;.$$

\end{theorem}

When comparing our new bound to the previous hard margin bound, i.e. every margin is at least $\theta$, note that the previous strongest results were Theorem~\ref{thm:bartlettNoErrors} and the bound in~\eqref{eq:mcallester} (setting $\Loss_S^\theta(w)=0$). Theorem~\ref{thm:mainupper} improves the former by a logarithmic factor and improves the additive $O\left(\sqrt{(\ln m + \ln(1/\delta))/m}\right)$ term in the latter to $O(\ln(1/\delta)/m)$. For soft margin the best known bounds are \eqref{eq:bartlettMendelsonWithErrors} and \eqref{eq:mcallester}. We improve over the former \eqref{eq:bartlettMendelsonWithErrors} for any choice of margin $\theta$ with $\Loss_S^\theta(w)< 1/\ln m$ and we improve over \eqref{eq:mcallester} once again by replacing the additive $O\left(\sqrt{(\ln m + \ln(1/\delta))/m}\right)$ term by $O(\ln(1/\delta)/m)$.

A natural question to ask is whether this new bound is close to optimal. In particular, for $\delta = \Omega(1)$, our new generalization bound simplifies to:
$$
\Loss_\D(w) \le \Loss_S^\theta(w) + O\left(\frac{R^2 \ln m}{\theta^2 m} + \sqrt{\frac{R^2 \ln m\cdot \Loss_S^\theta(w)}{\theta^2 m}}\right).
$$
and the generalization bound in~\eqref{eq:bartlettMendelsonWithErrors} becomes:
$$
\Loss_\D(w)\le \Loss_S^\theta(w) + O\left(\sqrt{\frac{R^2}{\theta^2 m}}\right).
$$
Summarizing the two, we get:
\begin{corollary}
\label{cor:combined}
Let $d \in \mathbb{N}^+$ and let $R > 0$. Denote by $X$ the ball of radius $R$ in $\mathbb{R}^d$ and let ${\cal D}$ be any distribution over $X \times \{-1,1\}$. Then it holds with constant probability over a set of $m$ samples $S \sim {\cal D}^m$, that for every $w \in \mathbb{R}^d$ with $\|w\|_2 \leq 1$ and every margin $\theta > 0$, we have

$$\Loss_\D(w) \le \Loss_S^\theta(w)  + O\left(\frac{R^2 \ln m}{\theta^2 m}  + \sqrt{\frac{R^2 }{\theta^2 m}  \cdot \min\{\ln m \cdot \Loss_S^\theta(w),1\} }\right) \;.$$
\end{corollary}
At first glance the bound presented in Corollary~\ref{cor:combined} might seem odd. The first expression inside the $O$-notation, which intuitively stands for the hard-margin bound, incorporates a $\ln m$ factor, while the second term, which intuitively stands for the soft-margin bound does not. 
Our second main result, however, demonstrates that Corollary~\ref{cor:combined} is in fact tight for most ranges of parameters. Specifically, one cannot remove the extra $\ln m$ factor for the hard-margin case.
\begin{theorem}
\label{thm:existMain}
There exists a universal constant $C>0$ such that for every $R \geq C \theta$, every $m \geq (R^2/\theta^2)^{1.001}$ and every $0 \leq \tau \leq 1$, there exists a distribution $\D$ over $X \times \{-1,+1\}$, where $X$ is the ball of radius $R$ in $\R^u$ for some $u$, such that with constant probability over a set of $m$ samples $S \sim \D^m$, there exists a vector $w$ with $\|w\|_2 \leq 1$ and $\Loss_S^\theta(w) \leq \tau$
satisfying:
$$
\Loss_{\cal D}(w) \geq \Loss_S^\theta(w) +  \Omega\left(\frac{R^2 \ln m}{\theta^2 m}+ \sqrt{\frac{R^2\ln(\tau^{-1})\tau}{\theta^2 m}} \right)\geq \Loss_S^\theta(w) +  \Omega\left(\frac{R^2 \ln m}{\theta^2 m}+ \sqrt{\frac{R^2\ln(\Loss_S^\theta(w)^{-1})\Loss_S^\theta(w)}{\theta^2 m}} \right) \;. 
$$
\end{theorem}

Together with Theorem~\ref{thm:existMain}, Corollary~\ref{cor:combined} gives the first completely tight generalization bounds in the hard margin case (by setting $\tau = 0$ in Theorem~\ref{thm:existMain}, and defining $0 \ln(0^{-1}) = 0$).
For the soft margin SVM case, the bounds are only off from one another by a factor $$\sqrt{\ln m/ \ln( \Loss_S^\theta(w)^{-1})}$$ i.e. they asymptotically match when $\Loss_S^\theta(w)\leq m^{-\eps}$ for an arbitrarily small constant $\eps>0$.
Our generalization lower bound also shows that the previous generalization bound in~\eqref{eq:bartlettMendelsonWithErrors} is tight when $\Loss_S^\theta(w) \geq \eps$ for any constant $\eps>0$.
Thus our main results settle the generalization performance of Support Vector Machines in terms of the classic margin-based parameters for all ranges of $\Loss_S^\theta(w) $ not including $m^{-o(1)} \leq \Loss_S^\theta(w)  \leq o(1)$.

We remark that our upper bound generalize to infinite dimension as it only depends on the ability for performing Johnson Lindenstrauss transforms of the data which works for Hilbert spaces in general \cite{JL84}.

We complement our existential lower bound with an algorithmic lower bound demonstrating limitations on the performance of any SVM learning algorithm. More specifically we show that for every algorithm, there exists a {\em reasonable} distribution for which the performance of the algorithm in terms of out of sample error are limited. We draw the reader's attention to the fact that the lower bound presented in Theorem~\ref{thm:existMain}, while precisely fitting the phrasing of classic upper bounds, as well as the upper bound presented in Theorem~\ref{thm:mainupper}, is purely existential, and does not rule out the existence of an algorithm that performs better than the 'adversarial' worst case. The next result thus gives a lower bound that employs a somewhat broader view.
Formally, given a learning algorithm $\A$, denote by $w_{\A,S}$ the hyperplane produced by $\A$ upon receiving sample set $S$. In these notations we show the following.
\begin{theorem} \label{th:lowerAlg}
For every large enough integer $N$, every $R \ge 1$, $\theta \in \left(1/N, 1/40\right)$ and $\tau \in [0,49/100]$ there exists an integer $k$ such that for every $m = \Omega\left(R^2/\theta^2\right)$, for every (randomized) learning algorithm $\A$, there exist a distribution $\D$ over the radius $R$ ball in $\mathbb{R}^k$ and $w \in \mathbb{R}^k$ such that $\|w\|_2=1$ and with probability at least $1/100$ over the choice of $(x_1,y_1),\ldots,(x_m,y_m) \sim \D^m$ and the random choices of $\A$
\begin{enumerate}
	\item $\Loss_S^\theta(w) < \tau$. 
	\item $\Loss_{\D}(w_{\A,S}) \ge \tau + \Omega\left(\frac{R^2}{m \theta^2} + \sqrt{\tau \cdot \frac{R^2}{m \theta^2}}\right)$.
\end{enumerate}
\end{theorem}
In order to get a better grasp of the theorem statement, we first turn to carefully analyze the two parts of the theorem, starting with the second, perhaps clearer out of sample error bound. Considering the second part of the theorem, it states that for any algorithm $\A$, there is a distribution $\D$ for which the out-of-sample error of the voting classifier produced by $\A$ is at least the given bound. The first part of the theorem ensures that at the same time, there exists a hyperplane $w$ obtaining a margin of at least $\theta$ on at least a $1-\tau$ fraction of the sample points. Our proof of Theorem~\ref{th:lowerAlg} not only shows that such $w$ exists, but also provides a specific construction. 
Loosely speaking, the first part of the theorem reflects on the nature of the distribution $\D$. Loosely speaking, the bound means that the distribution is not too hard, namely, it is possible to output a hyperplane $w$ with good margins. As the theorem gives a bound that holds for every algorithm, we cannot hope to prove that the first bound holds for $w_{\A,S}$, as we assume nothing on the performance of $\A$. Specifically, we cannot assume $\A$ attempts to optimize margins.
The second part of the theorem thus guarantees that regardless of which vector $w_{\A,S}$ the algorithm $\A$ produces, it still has large out-of-sample error. Specifically (but not limited to) every algorithm that minimizes the empirical risk, must have a large error.
Finally, comparing Theorem~\ref{th:lowerAlg} to Corollary~\ref{cor:combined}, if we associate $\tau$ with $\Loss_S^\theta(w_{\A,S})$. The magnitude of the out-of-sample error in the second point in Theorem~\ref{th:lowerAlg} thus matches that of Corollary~\ref{cor:combined}, except for a factor $\ln m$ in the first term inside the $\Omega(\cdot )$ and a $\sqrt{\min\{\ln m, 1/\Loss_S^\theta(w_{\A,S})\}}$ factor in the second term.
In conclusion, even when considering generalization bounds for specific SVM learning algorithms, there is not much room for improvement over our generalization upper bound given in Corollary~\ref{cor:combined}.

\section{Margin-Based Generalization Upper Bound}
This section is devoted to the proof of Theorem~\ref{thm:mainupper}, and we start by recollecting some notation. To this end, let $d \in \mathbb{N}^+$ and let $R, \delta > 0$. Let $\D$ be some distribution over $X \times \{-1,1\}$, where $X$ is the $R$-radius ball around the origin in $\mathbb{R}^d$, and let ${\cal H}$ denote the unit ball in $\mathbb{R}^d$. Finally, let ${\cal E} = {\cal E}(d,R,m,\delta) \subseteq (X \times \{-1,1\})^m$ include all sequences $S \in (X \times \{-1,1\})^m$ such that for every $w \in {\cal H}$ and $\theta > 0$, 
$$\Loss_{\cal D}(w) \le \Loss_S^\theta(w)  + O\left(\pi + \sqrt{\pi \Loss_S^\theta(w)}\right) \;,$$
where $\pi = \pi(\delta) = \frac{(R/\theta)^2 \ln m + \ln(1/\delta)}{m}$.
In these notations the theorem states that $\Pr_{S \sim \D^m}[{\cal E}] \ge 1-\delta$. 

\paragraph{Key Tools and Techniques.} One known method to prove such bounds (see, e.g. \cite{SFBL98, GZ13}) is to discretize the set of classifiers (or hyperplanes) and then union bound over the discrete set. When considering hyperplanes in $\mathbb{R}^d$, however, the discretization results in too large a set, which in turn means that the resulting union bound gives too large a probability bound. More specifically, the size of the set depends on the dimension $d$. 
In order to overcome this difficulty, and give generalization upper bound for a general $d$-dimensional distribution $\D$ we first reduce the dimension of the data set to a small dimension while approximately maintaining the geometric structure of the data set. That is, the dot products of a set points $x \in X$ with hyperplanes $w \in {\cal H}$ are maintained by the projection with high probability. More specifically, we randomly project both balls $X$ and ${\cal H}$ onto a small dimension $k$, while approximately preserving the inner products. The random linear projection we use is simply a matrix whose every entry is sampled independently from a standard normal distribution. While this projection matrix has been studied in previous applications of dimensionality reduction such as the Johnson-Lindenstrauss transform \cite{JL84,DG03}, we present some new analysis and give tight bounds that show that inner product values in $X \times {\cal H}$ are well-preserved with high probability by the projection.
We next discretize the set of hyperplanes in $\mathbb{R}^k$, using techniques inspired by \cite{AK17}, and show that it is enough to union bound over the resulting small grid.

We now turn to prove the theorem. Note first that if $\theta > R$ then the bound is trivial, since for every $S$, $\Pr_{(x,y)\sim S}\left[y\dotp{x}{w} \le \theta \right] = 1$. We may therefore assume hereafter that $\theta \in (0,R]$. Similarly we assume that $m \ge (R/\theta)^2\ln m + \ln(1/\delta)$. To show that ${\cal E}$ occurs with high probability, we next define a sequence $\{{\cal E}_k\}_{k \in \mathbb{N}^+}$ of events whose intersection is contained in ${\cal E}$ and has probability at least $1-\delta$.  
In order to define the sequence $\{{\cal E}_k\}_{k \in \mathbb{N}^+}$ we start by defining, for every $w \in {\cal H}$ and every positive integer $k \in \mathbb{N}^+$, a distribution ${\cal Q}_k(w)$ over $\mathbb{R}^d \to \mathbb{R}$. Loosely speaking, every function $g \in \supp({\cal Q}_k(w))$ takes a vector $x \in \mathbb{R}^d$, projects it into $\mathbb{R}^k$ and then takes its inner product with a vector $\tilde{w} \in \mathbb{R}^k$. The vector $\tilde{w}$ is the projection of $w$ into $\mathbb{R}^k$ rounded to a predefined grid in $\mathbb{R}^k$. 
Formally, we next describe the process that samples $g \sim {\cal Q}_k(w)$. First sample a projection matrix $A \in \mathbb{R}^{k \times d}$ from $\mathbb{R}^d$ to $\mathbb{R}^k$. Every entry of $A$ is independently sampled from a normal distribution ${\cal N}(0,1/k)$ with mean $0$ and variance $1/k$. 
Next, we define the vector $\tilde{w}$, which is a randomized rounding of $Aw$ to the grid of vectors in $\mathbb{R}^k$ whose every entry is a whole multiple of $1/\sqrt{k}$. For every $j \in [k]$, let $\ell$ be the unique integer such that $\ell \le \sqrt{k}[Aw]_j < \ell+1$. Set $\tilde{w}_j = \ell/\sqrt{k}$ with probability $(\ell+1) - \sqrt{k}[Aw]_j$ and $\tilde{w}_j = (\ell+1)/\sqrt{k}$ otherwise, independently for every $j \in [k]$ and independently of the choice of $A$. 
Finally, define $g : \mathbb{R}^d \to \mathbb{R}$ by $g(x) = \dotp{Ax}{\tilde{w}}$ for every $x \in \mathbb{R}^d$. For every $w \in {\cal H}$ and every $g \in \supp({\cal Q}_k(w))$ denote by $A_g \in \mathbb{R}^{k \times d}$ the matrix associated with $g$. Note that the choice of $A_g$ does not depend on $w$. If $w$ is clear from context we simply write ${\cal Q}_k$ instead of ${\cal Q}_k(w)$.

Finally, for every $k \in \mathbb{N}^+$, let $\Delta_k$ be the set of all vectors $v \in \mathbb{R}^k$ satisfying that $\|v\|_2^2 \le 6$ and for every $j \in [k]$, $v_j \sqrt{k}$ is an integer. We are now ready to define the sequence $\{{\cal E}_k\}_{k \in \mathbb{N}^+}$ of events.
\begin{definition}
Let $k \in \mathbb{N}^+$. For every $A \in \mathbb{R}^{k \times d}$ and $S \in \supp({\cal D}^m)$, we say that $A$ and $S$ are {\em compatible} if for all $v \in \Delta_k$ and $\ell \in [10k]$, 

\begin{equation}
\begin{split}
\Pr_{(x,y) \sim {\cal D}}[y \dotp{Ax}{v} \le \ell R/(10k)] &\le \Pr_{(x,y) \sim S}[y\dotp{Ax}{v} \le \ell R/(10k)] \\
&+ \frac{8\ln(2^{9k}/\delta)}{m} + 4 \sqrt{\Pr_{(x,y) \sim S}[y\dotp{Ax}{v} \le \ell R/(10k)]\cdot\frac{\ln(2^{9k}/\delta)}{m}} \;.
\end{split}
\label{eq:compatible}
\end{equation}

Let ${\cal C}$ denote the set of all compatible pairs $(A,S)$. Finally, let ${\cal E}_k$ be the set of all $S \in \supp({\cal D}^m)$ such that for all $w \in {\cal H}$, $\Pr_{g \sim {\cal Q}_k}[(A_g,S)\in {\cal C}] \ge 1- 6 \cdot 2^{-k/2}$.
\end{definition}
The next lemma implies Theorem~\ref{thm:mainupper} by simply applying a union bound, since $\sum_k{\frac{1}{k(k+1)}}=1$.
\begin{lemma} \label{l:mainTechLemma}
For every $k \in \mathbb{N}^+$, $\Pr_{S \sim {\cal D}^m}[{\cal E}_k] \ge 1-\frac{\delta}{k(k+1)}$, and moreover $\bigcap_{k \in \mathbb{N}^+}{\cal E}_k \subseteq {\cal E}$.
\end{lemma} 
We start by proving that for every $k$, with high probability over $S \sim {\cal D}^m$, $S \in {\cal E}_k$. The first step is to prove that for every fixed matrix $A$, a random sample $S \sim {\cal D}^m$ is compatible with $A$ with very high probability. Using Markov's inequality we then conclude that a random sample $S \sim {\cal D}^m$ is, with very high probability, compatible with most projection matrices $\{A_g\}_{g \in \supp({\cal Q}_k(w)}$ for every $w \in {\cal H}$. Formally, we prove the following.
\begin{claim} \label{c:compatibleFixedMatrix}
For every $A \in \mathbb{R}^{d \times k}$, $\Pr_{S \sim {\cal D}^m}[(A,S) \in {\cal C}] \ge 1 - \delta/2^k$.
\end{claim}
\begin{proof}
Let $A \in \mathbb{R}^{d \times k}$, and fix some $v \in \Delta_k$ and $\ell \in [10k]$. First note that if $\Pr_{(x,y)\sim {\cal D}}[y\dotp{Ax}{v}\le \ell R /(10k)] \le \frac{8 \ln (2^{9k}/\delta)}{m}$ then \eqref{eq:compatible} holds for all $S \in \supp({\cal D}^m)$. We can therefore assume that $\Pr_{(x,y)\sim {\cal D}}[y\dotp{Ax}{v}\le \ell R /(10k)] > \frac{8 \ln (2^{9k}/\delta)}{m}$. Let $\gamma = \sqrt{\frac{2\ln(2^{9k}/\delta)}{m\Pr_{(x,y)\sim {\cal D}}[y\dotp{Ax}{v}\le \ell R /(10k)]}}$, then $\gamma \in (0,1/2)$, and therefore a Chernoff bound then gives the following two inequalities.
\begin{equation}
\begin{split}
\Pr_{S \sim {\cal D}^m}&\left[\Pr_{(x,y) \sim S}\left[y\dotp{Ax}{v} \le \ell R /(10k)\right] < (1-\gamma)\Pr_{(x,y)\sim{\cal D}}\left[y\dotp{Ax}{v} \le \ell R /(10k)\right]\right] \\
&\le e^{-(m\gamma^2/2)\Pr_{(x,y)\sim{\cal D}}\left[y\dotp{Ax}{v} \le \ell R /(10k)\right]} = \frac{\delta}{2^{9k}}
\end{split}
\label{eq:lowerChernoff}
\end{equation}

\begin{equation}
\begin{split}
\Pr_{S \sim {\cal D}^m}&\left[\Pr_{(x,y) \sim S}\left[y\dotp{Ax}{v} \le \ell R /(10k)\right] > 2\Pr_{(x,y)\sim{\cal D}}\left[y\dotp{Ax}{v} \le \ell R /(10k)\right]\right] \\
&\le e^{-(m/2)\Pr_{(x,y)\sim{\cal D}}\left[y\dotp{Ax}{v} \le \ell R /(10k)\right]} \le \frac{\delta}{2^{9k}} \;,
\end{split}
\label{eq:upperChernoff}
\end{equation}
where the last inequality is due to the fact that $(m/2)\Pr_{(x,y)\sim{\cal D}}\left[y\dotp{Ax}{v} \le \ell R /(10k)\right] \ge \ln(2^{9k}/\delta)$.

Hence with probability at least $1 - 2\delta/2^{9k}$ over the choice of $S$ we have that
\begin{equation}
\begin{split}
\Pr_{(x,y)\sim{\cal D}}\left[y\dotp{Ax}{v} \le \ell R /(10k)\right] &\le (1-\gamma)^{-1}\Pr_{(x,y) \sim S}\left[y\dotp{Ax}{v} \le \ell R /(10k)\right] \\
&\le (1+2\gamma)\Pr_{(x,y) \sim S}\left[y\dotp{Ax}{v} \le \ell R /(10k)\right]\;,
\end{split}
\label{eq:upperBound}
\end{equation}
and moreover, 
\begin{equation}
\begin{split}
\gamma = \sqrt{\frac{2\ln(2^{9k}/\delta)}{m\Pr_{(x,y)\sim {\cal D}}[y\dotp{Ax}{v}\le \ell R /(10k)]}} \le \sqrt{\frac{4\ln(2^{9k}/\delta)}{m\Pr_{(x,y)\sim S}[y\dotp{Ax}{v}\le \ell R /(10k)]}}
\end{split}
\label{eq:upperBoundGamma}
\end{equation}
Plugging \eqref{eq:upperBoundGamma} into \eqref{eq:upperBound} and summing up we get that for every $v \in \Delta_k$ and $\ell \in [10k]$, with probability at least $1 -2\delta/2^{9k}$ over the choice of $S$ we have
\begin{equation}
\begin{split}
\Pr_{(x,y)\sim{\cal D}}\left[y\dotp{Ax}{v} \le \ell R /(10k)\right] &\le \Pr_{(x,y) \sim S}\left[y\dotp{Ax}{v} \le \ell R /(10k)\right] + \frac{8\ln(2^{9k}/\delta)}{m}\\
&+ 4 \sqrt{\frac{\ln(2^{9k}/\delta)}{m}\Pr_{(x,y) \sim S}\left[y\dotp{Ax}{v} \le \ell R /(10k)\right]}
\end{split}
\label{eq:upperBoundOneVecOneMarg}
\end{equation}
Union bounding over all $v \in \Delta_k$ and $\ell \in [10k]$ we get that $\Pr_{S\sim{\cal D}^m}[(A,S)\in{\cal C}] \ge 1 - 10k|\Delta_k|\delta/2^{9k}$. To finish the proof of the claim, we show that $|\Delta_k| \le 2^{6k}$.
Let $v \in \Delta_k$, then as $|v_j\sqrt{k}| \in \mathbb{N}$ for all $j \in [k]$ then $\sum_{j \in [k]}{|v_j\sqrt{k}|} \le \sum_{j \in [k]}{|v_j\sqrt{k}|^2} \le 6k$. Therefore the number of possible ways to construct $|v_1\sqrt{k}|, \ldots, |v_k\sqrt{k}|$ is the number of possible solutions to the equation $\sum_{j \in [k+1]}{x_j} = 6k$ in natural numbers, which is $\binom{7k}{6k} \le 2^{4.5k}$. Taking all possible signs into account gives $|\Delta_k| \le 2^{5.5k}$.
We conclude that $\Pr_{S\sim{\cal D}^m}[(A,S)\in{\cal C}] \ge 1 - \delta/2^{k}$.
\end{proof}

\begin{corollary}
$\Pr_{S \sim {\cal D}^m}[{\cal E}_k] \ge 1 - \delta/(k(k+1))$.
\end{corollary}
\begin{proof}
Fix some $w_0 \in {\cal H}$. Note that for every $S \in \supp({\cal D}^m)$, if it holds that $\Pr_{g \sim {\cal Q}_k(w_0)}[(A_g,S) \in {\cal C}] \ge 1- 6 \cdot 2^{-k/2}$, then it is true that for all $w \in {\cal H}$, $\Pr_{g \sim {\cal Q}_k(w_0)}[(A_g,S) \in {\cal C}] \ge 1 - 6 \cdot 2^{-k/2}$, as the choice of $A_g$ does not depend on $w$.
From Claim~\ref{c:compatibleFixedMatrix} we conclude that 
$$\mathbb{E}_{S \sim {\cal D}^m}\left[\Pr_{g \sim {\cal Q}_k(w_0)}[(A_g,S)\in {\cal C}]\right] = \mathbb{E}_{g \sim {\cal Q}_k(w_0)}\left[\Pr_{S \sim {\cal D}^m}[(A_g,S)\in {\cal C}]\right] \ge 1 - \delta/2^k\;.$$
From Markov's inequality, and since for every $k \in \mathbb{N}^+$, $k(k+1) \le 6 \cdot 2^{k/2}$ we conclude that 
$$\Pr_{S \sim {\cal D}^m}[{\cal E}_k] \ge \Pr_{S \sim {\cal D}^m}\left[\Pr_{g \sim {\cal Q}_k(w_0)}[(A_g,S) \in {\cal C}] \ge 1-\frac{k(k+1)}{2^k}\right]\ge 1 - \delta/((k(k+1))\;.$$
\end{proof}
We next prove the second part of Lemma~\ref{l:mainTechLemma}, namely that $\bigcap_{k \in \mathbb{N}^+}{\cal E}_k \subseteq {\cal E}$.
We start by introducing some concentration bounds on sums of products of Gaussian random variables.
\begin{lemma} \label{l:jlDot}
Let $A \in \mathbb{R}^{d \times k}$ be a matrix whose every entry is
independently ${\cal N}(0,1/k)$ distributed. Then for every $u,v \in
\mathbb{R}^d$ and $t \in [0, 1/4)$ we have 
\begin{enumerate}
	\item $\Pr_A[|\|Au\|_2^2 - \|u\|_2^2| > t\|u\|_2^2] \le 2e^{-0.21kt^2}$; and
	\item $\Pr_A[|\dotp{Au}{Av} - \dotp{u}{v}| > t] \le 4e^{\frac{-kt^2}{7\|u\|_2^2\|v\|_2^2}}\;.$
\end{enumerate}
\end{lemma}
The proof of the lemma is quite technically involved, and its proof is thus deferred to Appendix~\ref{appSec:jlDot}. 
The next claim shows that with very high probability over the choice of a pair $(x,y)$, either sampled from ${\cal D}$ or uniformly at random from a sample $S$, and the choice of $g \sim {\cal Q}_k(w)$, the values $\dotp{x}{w}$ and $g(x)$ cannot be too far apart.

\begin{claim} \label{c:moveToGrid}
For all $w \in {\cal H}, \theta \in (0,R]$ and $k \in \mathbb{N}^+$, 
\begin{enumerate}
	\item $\Pr\limits_{(x,y)\sim {\cal D}, g \sim {\cal Q}_k}[y\dotp{x}{w} \le 0 \wedge yg(x) \ge 49\theta/100 ] \le 7e^{-\left(\frac{k}{120}\right)\left(\frac{\theta}{R}\right)^2}$; and
	\item For every $S \in \supp({\cal D}^m)$,  \newline $\Pr\limits_{(x,y)\sim S, g \sim {\cal Q}_k}[y\dotp{x}{w} \ge \theta \wedge yg(x) \le \theta/2] \le 7e^{-\left(\frac{k}{120}\right)\left(\frac{\theta}{R}\right)^2}\;.$
\end{enumerate}
\end{claim}

\begin{proof}
Let $w \in {\cal H}$, $\theta>0$ and $k \in \mathbb{N}^+$. Then  
\begin{equation*}
\Pr_{(x,y)\sim {\cal D}, g \sim {\cal Q}_k}[y\dotp{x}{w} \le 0 \wedge yg(x) \ge 49\theta/100 ] \le \Pr_{(x,y)\sim {\cal D}, g \sim {\cal Q}_k}[| y\dotp{x}{w} - yg(x) | > 49\theta/100 ]
\end{equation*}
Recall that for every $x \in \mathbb{R}^d$, $g(x) = \dotp{Ax}{\tilde{w}}$, where every entry of $A \in \mathbb{R}^{d \times k}$ is sampled independently from a Gaussian distribution with mean $0$ and variance $1/k$, and $\tilde{w} \in \mathbb{R}^k$ is constructed by randomly rounding each entry of $Aw$ independently to a multiple of $1/\sqrt{k}$.
By the triangle inequality, the linearity of the dot product, and since $y \in \{-1,1\}$,
$$| y\dotp{x}{w} - yg(x) | \le | \dotp{x}{w} - \dotp{Ax}{Aw} | + | \dotp{Ax}{Aw-\tilde{w}} | \;.$$
Therefore 
\begin{equation}
\begin{split}
&\Pr_{(x,y)\sim {\cal D}, g \sim {\cal Q}_k}[y\dotp{x}{w} \le 0 \wedge yg(x) > 49\theta/100 ] \le  \Pr_{(x,y)\sim {\cal D}, g \sim {\cal Q}_k}[ |y\dotp{x}{w} - yg(x)| > 49\theta/100 ]\\
&\le \Pr_{(x,y)\sim {\cal D}, g \sim {\cal Q}_k}[| \dotp{x}{w} - \dotp{Ax}{Aw} | > 49\theta/200 ] + \Pr_{(x,y)\sim {\cal D}, g \sim {\cal Q}_k}[| \dotp{Ax}{Aw - \tilde{w}}| > 49\theta/200 ]
\end{split}
\label{eq:dotTriangle}
\end{equation}
To bound the first probability term observe that
\begin{equation}
\begin{split}
&\Pr\limits_{(x,y)\sim{\cal D},g \sim {\cal Q}_k}\left[\left| \dotp{x}{w} - \dotp{Ax}{Aw} \right| > \frac{49\theta}{200} \right] \\
&\le \mathbb{E}_{(x,y)\sim{\cal D}}\left[\Pr_{g \sim {\cal Q}_k}\left[\left| \frac{\dotp{x}{w}}{\|x\|_2\|w\|_2} - \frac{\dotp{Ax}{Aw}}{\|x\|_2\|w\|_2} \right| > \frac{49\theta}{200R}\right] \right] \\
&\le 4e^{-\frac{k}{7}\left(\frac{49\theta}{200R}\right)^2}\;,
\end{split}
\label{eq:jlDot}
\end{equation}
where the inequality before last follows from the fact that $\|w\|_2 \le 1$ and $\Pr_{(x,y)\sim{\cal D}}[\|x\|_2 \le R] = 1$, and the last inequality is an application of Lemma~\ref{l:jlDot}.

To bound the second term in \eqref{eq:dotTriangle}, fix $(x,y) \in \supp({\cal D})$ and $A \in \mathbb{R}^{k \times d}$, and denote $Aw = \hat{w}$. Then for every $j \in [k]$ independently $\tilde{w}_j = \frac{\left\lfloor\sqrt{k}\hat{w}_j\right\rfloor}{\sqrt{k}}$ with probability $\left\lfloor \sqrt{k}\hat{w}_j \right\rfloor + 1 - \sqrt{k}\hat{w}_j$, and $\tilde{w}_j = \frac{\left\lfloor\sqrt{k}\hat{w}_j\right\rfloor+1}{\sqrt{k}}$ otherwise. Therefore for every $j \in [k]$, 
$$\mathbb{E}[\tilde{w}_j] = \frac{\left\lfloor\sqrt{k}\hat{w}_j\right\rfloor}{\sqrt{k}}(\left\lfloor \sqrt{k}\hat{w}_j \right\rfloor + 1 - \sqrt{k}\hat{w}_j) + \frac{\left\lfloor\sqrt{k}\hat{w}_j\right\rfloor+1}{\sqrt{k}}(\sqrt{k}\hat{w}_j - \left\lfloor \sqrt{k}\hat{w}_j \right\rfloor) = \hat{w}_j\;,$$
and thus $\mathbb{E}[\dotp{Ax}{Aw-\tilde{w}}] = 0$. A Hoeffding bound then yields
$$\Pr_{g \sim {\cal Q}_k}[| \dotp{Ax}{Aw-\tilde{w}}| > 49\theta/200 \mid A_g=A] \le 2e^{\frac{-2(49\theta/200)^2}{\sum_{j \in [k]}{[Ax]_j^2(\hat{w}_j-\tilde{w}_j)^2}}} \le 2e^{-2k \left(\frac{49\theta}{200\|Ax\|_2}\right)^2} \;.$$
In addition,
\begin{equation*}
\Pr_{(x,y)\sim {\cal D}, g \sim {\cal Q}_k}[\|Ax\|_2 > \sqrt{1.25}R] \le \Pr_{(x,y)\sim {\cal D}, g \sim {\cal Q}_k}[\|Ax\|_2^2 - \|x\|_2^2 > 0.25\|x\|_2^2] \le e^{ - 0.21 \cdot 0.25^2 \cdot k} \le e^{-k/80}
\end{equation*}
Finally, we get that
\begin{equation}
\begin{split}
&\Pr_{(x,y)\sim {\cal D}, g \sim {\cal Q}_k}\left[\left| \dotp{Ax}{Aw-\tilde{w}}\right| > \frac{49\theta}{200} \right] \\
&\le \Pr_{(x,y)\sim {\cal D}, g \sim {\cal Q}_k}\left[\left| \dotp{Ax}{Aw-\tilde{w}}\right| > \frac{49\theta}{200} \left| \|A_gx\|_2 \le \sqrt{1.25}R\right.\right] + \Pr_{(x,y)\sim {\cal D}, g \sim {\cal Q}_k}[\|Ax\|_2 > \sqrt{1.25}R]\\
&\le 2e^{-2k\left(\frac{49\theta}{200\sqrt{1.25}R}\right)^2} + e^{-k/80} 
\end{split}
\label{eq:randRound}
\end{equation}
Plugging \eqref{eq:jlDot} and \eqref{eq:randRound} into \eqref{eq:dotTriangle} we get that
\begin{equation*}
\Pr_{\substack{(x,y)\sim {\cal D} \\ g \sim {\cal Q}_k}}[y\dotp{x}{w} \le 0 \wedge yg(x) > \theta/2 ]\le 7e^{-\left(\frac{k}{120}\right)\left(\frac{\theta}{R}\right)^2}\;,
\end{equation*}
which concludes the first part of the lemma. The proof of the second part is identical, as we did not use any property of the distribution ${\cal D}$ other than the fact that $\Pr_{(x,y)\sim{\cal D}}[\|x\|_2\le R] = 1$. For every $S \in \supp({\cal D}^m)$, it holds that $\Pr_{(x,y)\sim S}[\|x\|_2\le R] = 1$, and the result follows.
\end{proof}

The next claim essentially shows that restricting the definition of compatibility of a sample $S$ and a matrix $A$ only to grid points in $\Delta_k$ was indeed enough. Intuitively this is due to the fact that with very high probability over the choice of $q \sim {\cal Q}_k(w)$, the rounding of $A_gw$ is in the grid. Formally, we show the following.
\begin{claim} \label{c:DToS}
For every $S \in \bigcap_{k \in \mathbb{N}}{{\cal E}_k}$, for all $w \in {\cal H}, \theta \in (0,R]$ and $k \in \mathbb{N}^+$, 
	\begin{equation}
	\begin{split}
	\Pr\limits_{(x,y)\sim{\cal D},g \sim {\cal Q}_k}&[yg(x) \le 49\theta/100] \le \Pr\limits_{(x,y)\sim S, g \sim {\cal Q}_k}[yg(x) \le \theta/2] + 7e^{-\left(\frac{k}{120}\right)\left(\frac{\theta}{R}\right)^2} + 30e^{-k/24}\\
	&+ O\left(\frac{k + \ln(1/\delta)}{m} + \sqrt{\frac{k + \ln(1/\delta)}{m} \cdot \Pr\limits_{(x,y)\sim S, g \sim {\cal Q}_k}[yg(x) \le \theta/2]}\right)\;\; ;
	\end{split}
	\label{eq:2ndEvent}
	\end{equation}
\end{claim}
\begin{proof}
Fix $S \in \bigcap_{k \in \mathbb{N}}{{\cal E}_k}$, $w \in {\cal H}$, $\theta \in (0,R]$ and $k \in \mathbb{N}^+$. Clearly, if $\theta \le 10R/k$ then $7e^{-\left(\frac{k}{120}\right)\left(\frac{\theta}{R}\right)^2} \ge 1$ and therefore \eqref{eq:2ndEvent} holds. Otherwise, let $\ell$ be the smallest integer such that $49\theta/100 \le \ell R/(10k)$. As $\theta \le R$, $\ell \in [10k]$. In addition,  $49\theta/100 \le \ell R / (10k) \le 49\theta/100 + R/(10k)\le \theta/2$. Denote by ${\cal F}$ the event that $(A_g, S) \in {\cal C}$ and $\tilde{w} \in \Delta_k$ (recall that $\tilde{w}$ is the vector $Aw$, where each entry is rounded to the nearest multiple of $1/\sqrt{k}$). Hence
\begin{equation}
\begin{split}
\Pr_{(x,y)\sim{\cal D}, g \sim {\cal Q}_k}&[yg(x) \le 49\theta/100] \le \Pr_{(x,y)\sim{\cal D}, g \sim {\cal Q}_k}[yg(x) \le \ell R/(10k)] \\
&\le \Pr_{(x,y)\sim{\cal D}, g \sim {\cal Q}_k}[yg(x) \le \ell R/(10k) \mid {\cal F}] + \Pr_{g \sim {\cal Q}_k}[\bar{\cal F}]\\
&\le \mathbb{E}_{g \sim {\cal Q}_k}\left[\left. \Pr_{(x,y) \sim {\cal D}}[y g(x) \le \ell R/(10k)] \right| {\cal F}\right] + \Pr_{g \sim {\cal Q}_k}[\bar{\cal F}] \;,
\end{split}
\label{eq:roundMargin}
\end{equation}
By the definition of compatible pairs and linearity of expectation we get that 
\begin{equation*}
\begin{split}
\mathbb{E}_{g \sim {\cal Q}_k}&\left[\left. \Pr_{(x,y) \sim {\cal D}}[y g(x) \le \ell R/(10k)] \right| {\cal F}\right] \le \mathbb{E}_{g \sim {\cal Q}_k}\left[\left. \Pr_{(x,y) \sim S}[yg(x) \le \ell R/(10k)] \right| {\cal F}\right] \\
&+ \frac{8\ln(2^{9k}/\delta)}{m} + 4 \mathbb{E}_{g \sim {\cal Q}_k}\left[\left.\sqrt{\Pr_{(x,y) \sim S}[yg(x) \le \ell R/(10k)]\cdot\frac{\ln(2^{9k}/\delta)}{m}} \right| {\cal F}\right]\;.
\end{split}
\end{equation*}
Note that for every non-negative random variable $Y$ and event $E$, $\mathbb{E}[Y|E] \le \mathbb{E}[Y]/\Pr[E]$. We therefore turn to bound the probability of ${\cal F}$. By a simple union bound, 
$$\Pr_{g \sim{\cal Q}_k}[\bar{{\cal F}}] \le \Pr_{g \sim {\cal Q}_k}[(A_g,S)\notin {\cal C}] + \Pr_{g \sim {\cal Q}_k}[\tilde{w} \notin \Delta_k]\;.$$ 
Since $S \in {\cal E}_k$, $\Pr_{g \sim {\cal Q}_k}[(A_g,S)\notin {\cal C}] \le 6 \cdot 2^{-k/2}$. Next, for every $j \in [k]$, $|\tilde{w}_j| \le |[A_gw]_j| + 1/\sqrt{k}$. Therefore $\|\tilde{w}\|_2^2 \le \|A_gw\|_2^2 + 1 + 2\max\{\|A_gw\|_2^2, 1\}$, and hence if $\|Aw\|_2^2 \le 1.5$, then $\|\tilde{w}\|_2^2 \le 6$, and therefore $\tilde{w} \in \Delta_k$. We conclude that $\Pr_{g \sim {\cal Q}_k}[\tilde{w} \notin \Delta_k] \le \Pr_{g \sim {\cal Q}_k}[\|A_gw\|_2^2 > 1.5] \le e^{-k/24} \;,$
and hence $\Pr_{g \sim{\cal Q}_k}[{\cal F}] \ge 1 - 7e^{-k/24} \ge (1 + 15e^{-k/24})^{-1}$. Since, in addition, $\ell R/(10k) \le \theta/2$ we get
\begin{equation*}
\begin{split}
\mathbb{E}_{g \sim {\cal Q}_k}&\left[\left. \Pr_{(x,y) \sim {\cal D}}[y g(x) \le \ell R/(10k)] \right| {\cal F}\right] \le (1 + 15e^{-k/24})\mathbb{E}_{g \sim {\cal Q}_k}\left[\Pr_{(x,y) \sim S}[yg(x) \le \theta/2] \right] \\
&+ \frac{8\ln(2^{9k}/\delta)}{m} + 4 (1 + 15e^{-k/24}) \mathbb{E}_{g \sim {\cal Q}_k}\left[\sqrt{\Pr_{(x,y) \sim S}[yg(x) \le \theta/2]\cdot\frac{\ln(2^{9k}/\delta)}{m}} \right] \;.
\end{split}
\end{equation*}
Finally, by Jensen's inequality we get
\begin{equation*}
\begin{split}
\mathbb{E}_{g \sim {\cal Q}_k}&\left[\left. \Pr_{(x,y) \sim {\cal D}}[y g(x) \le \ell R/(10k)] \right| (A_g,S) \in {\cal C}\right] \le \mathbb{E}_{g \sim {\cal Q}_k}\left[\Pr_{(x,y) \sim S}[yg(x) \le \theta/2] \right] \\
&+ \frac{8\ln(2^{9k}/\delta)}{m} + 4 \sqrt{\mathbb{E}_{g \sim {\cal Q}_k}\left[\Pr_{(x,y) \sim S}[yg(x) \le \theta/2]\right]\cdot\frac{\ln(2^{9k}/\delta)}{m}} +30 e^{-k/24}\;.
\end{split}
\end{equation*}
Plugging into \eqref{eq:roundMargin} we get \eqref{eq:2ndEvent}.
\end{proof}
To finish the proof of Lemma~\ref{l:mainTechLemma}, let $S = \left\langle(x_j,y_j)\right\rangle_{j \in [m]}\in \bigcap_{k \in \mathbb{N}^+}{\cal E}_k$, fix some $w \in {\cal H}$ and $\theta > 0$, and let $k = \left\lceil 240\left(\frac{R}{\theta}\right)^2\ln m\right\rceil$. We will show that $S \in {\cal E}$.
\begin{equation}
\begin{split}
\Pr_{(x,y)\sim {\cal D}}&[y\dotp{x}{w} \le 0] = \Pr_{(x,y)\sim {\cal D}, g \sim {\cal Q}_k}[y\dotp{x}{w} \le 0] \\
&\le \Pr_{(x,y)\sim {\cal D}, g \sim {\cal Q}_k}[yg(x) \le 49\theta/100] + \Pr_{(x,y)\sim {\cal D}, g \sim {\cal Q}_k}[y\dotp{x}{w} \le 0 \wedge yg(x) > 49\theta/100 ] \\
&\le \Pr_{(x,y)\sim {\cal D}, g \sim {\cal Q}_k}[yg(x) \le 49\theta/100] + \frac{1}{m} \\
\end{split}
\label{eq:introduceG}
\end{equation}
Where the last inequality is due to Claim~\ref{c:moveToGrid}, and since $7e^{-\left(\frac{k}{120}\right) \left(\frac{\theta}{R}\right)^2} \le 7/m^2 \le 1/m \;.$
From Claim~\ref{c:DToS} we get 
\begin{equation}
	\begin{split}
	\Pr\limits_{(x,y)\sim{\cal D},g \sim {\cal Q}_k}[yg(x) \le 49\theta/100] &\le \Pr\limits_{(x,y)\sim S, g \sim {\cal Q}_k}[yg(x) \le \theta/2] \\
	&+ O\left(\frac{k + \ln(1/\delta)}{m} + \sqrt{\frac{k + \ln(1/\delta)}{m} \cdot \Pr\limits_{(x,y)\sim S, g \sim {\cal Q}_k}[yg(x) \le \theta/2]}\right)\;\; ;
	\end{split}
	\label{eq:DToS}
	\end{equation}
Similarly to \eqref{eq:introduceG} we get that
\begin{equation}
\begin{split}
\Pr\limits_{(x,y)\sim S, g \sim {\cal Q}_k}&[yg(x) \le \theta/2] \\
&\le \Pr_{(x,y)\sim S, g \sim {\cal Q}_k}[y\dotp{x}{w} < \theta] + \Pr_{(x,y)\sim S, g \sim {\cal Q}_k}[y\dotp{x}{w} \ge \theta \wedge yg(x) \le \theta/2 ] \\
&\le \Pr_{(x,y)\sim S}[y\dotp{x}{w} < \theta] + \frac{1}{m^2}  \;.
\end{split}
\label{eq:elimG}
\end{equation}
Where the last inequality follows from Claim~\ref{c:moveToGrid} and the fact that $y\dotp{x}{w} \le \theta$ is independent of $g$.
Finally, plugging \eqref{eq:elimG} into \eqref{eq:DToS} and then into \eqref{eq:introduceG}, and assuming that $k + \ln(1/\delta) \le m$ we get that
\begin{equation*}
\begin{split}
\Pr_{(x,y)\sim {\cal D}}[y\dotp{x}{w} \le 0] &\le \Pr_{(x,y)\sim S}[y\dotp{x}{w} < \theta]+ \frac{1}{m} \\
&+ O\left(\frac{k + \ln(1/\delta)}{m} + \sqrt{\frac{k + \ln(1/\delta)}{m} \cdot \left(\Pr_{(x,y)\sim S}[y\dotp{x}{w} < \theta] + \frac{1}{m^2}\right)}\right)\\
&\le \Pr_{(x,y)\sim S}\left[y\dotp{x}{w} < \theta \right]  + O\left(\pi + \sqrt{\pi \cdot\Pr_{(x,y)\sim S}\left[y\dotp{x}{w} < \theta \right]}\right) \;,
\end{split}
\end{equation*}
where $\pi = \frac{(R/\theta)^2\ln m + \ln(1/\delta)}{m}$, and therefore $S \in {\cal E}$, and the proof of Lemma~\ref{l:mainTechLemma}, and thus of Theorem~\ref{thm:mainupper}, is now complete.

\section{Existential Lower Bound}
The goal of this section is to prove the generalization lower bound in Theorem~\ref{thm:existMain}. 
Our proof is split into two cases, depending on the magnitude of $\tau$. The results we prove are as follows:
\begin{lemma}
\label{thm:existSmallTau}
There is a universal constant $C>0$ such that for every $R \geq C\theta$ and every $m \geq (R^2/\theta^2)^{1.001}$, there exists a distribution $\D$ over $X \times \{-1,+1\}$, where $X$ is the ball of radius $R$ in $\R^u$ for some $u$, such that with constant probability over a set of $m$ samples $S \sim \D^m$, there exists a vector $w$ with $\|w\|_2 \leq 1$ and $\Loss_S^\theta = 0$
satisfying $\Loss_{\cal D} \geq \Omega\left( \frac{R^2 \ln m}{\theta^2 m} \right) \;.$
\end{lemma}

\begin{lemma}
\label{thm:existLargeTau}
There is a universal constant $C>0$ such that for every $R \geq C \theta$, every $m \geq (R^2/\theta^2)^{1.001}$ and every $R^2 \ln(m)/(\theta^2 m) < \tau \leq 1$, there exists a distribution $\D$ over $X \times \{-1,+1\}$, where $X$ is the ball of radius $R$ in $\R^{u}$ for some $u$, such that with constant probability over a set of $m$ samples $S \sim \D^m$, there exists a vector $w$ with $\|w\|_2 \leq 1$ and $\Loss_S^\theta \leq \tau$ satisfying 
\begin{equation*}
\Loss_{\D} \geq \Loss_S^\theta + \Omega\left(\sqrt{\frac{R^2\tau\ln\left(\tau^{-1}\right)}{\theta^2m}}\right) \;.
\end{equation*}
\end{lemma}
We will first show how to combine Lemma~\ref{thm:existSmallTau} and Lemma~\ref{thm:existLargeTau} to obtain Theorem~\ref{thm:existMain}. For any $0 \leq \tau \leq 1$, every $R \geq C\theta$ for a large constant $C>0$ and every $m  \geq (R^2/\theta^2)^{1.001}$, we can invoke Lemma~\ref{thm:existSmallTau} or Lemma~\ref{thm:existLargeTau} to conclude the existence of a distribution $\D$, such that with constant probability over a choice of $m$ samples $S \sim \D^m$, there is a vector $w$ with $\|w\|_2 \leq 1$ and either:
\begin{enumerate}
\item $\Loss_S^\theta(w) = 0 < \tau$ and 
$$
\Loss_{\D}(w) \geq \Loss_S^\theta(w)  + \Omega(R^2 \ln m/(\theta^2 m)).$$
\item $\Loss_S^\theta(w) \leq \tau$ and 
$$
\Loss_{\D}(w) \geq \Loss_S^\theta(w) + \Omega(\sqrt{(R^2/\theta^2)\ln(\tau^{-1})\tau /m}).$$
\end{enumerate} 
Note that Lemma~\ref{thm:existLargeTau} strictly speaking cannot be invoked for $\tau \leq R^2 \ln(m)/(\theta^2 m)$, but for such small values of $\tau$, the expression $\sqrt{(R^2/\theta^2)\ln(\tau^{-1})\tau /m}$ becomes less than $R^2 \ln m/(\theta^2 m)$ and the bound follows from Lemma~\ref{thm:existSmallTau} instead. Thus for any $0 \leq \tau \leq 1$, with constant probability over $S$, we may find a $w$ with $\Loss_S^\theta(w) \leq \tau$ and 
$$\Loss_{\D}(w) \le \Loss_S^\theta(w) + \Omega\left(\max\left\{R^2 \ln m/(\theta^2 m), \sqrt{(R^2/\theta^2)\ln(\tau^{-1})\tau /m}\right\}\right)\;,$$
and therefore 
$$\Loss_{\D}(w) \le \Loss_S^\theta(w) + \Omega\left(R^2 \ln m/(\theta^2 m) + \sqrt{(R^2/\theta^2)\ln(\tau^{-1})\tau /m}\right)\;.$$
This concludes the proof of Theorem~\ref{thm:existMain}. The following two sections prove the two lemmas.
\ifpdf
\subsection{Small \texorpdfstring{$\tau$}{t}}
\else
\subsection{Small $\tau$}
\fi
In this section, we prove Lemma~\ref{thm:existSmallTau}.
Let $m$ be the number of samples and assume $m \geq (R^2/\theta^2)^{1+\eps}$ where $\eps = 0.001$. Assume furthermore that $R \geq C\theta$ for a sufficiently large constant $C>0$. We construct a distribution $\D$ over $\R^{u+1} \times \{-1,+1\}$, where $u = 4e \eps^{-1} m/\ln m$. The distribution $\D$ gives a uniform random point among $\{x_1,\dots,x_u\}$ where $x_i$ has its $(u+1)$'st and $i$'th coordinate equal to $R/\sqrt{2}$ and the rest $0$. The label is always $1$.

Inspired by ideas by Gr{\o}nlund \etal \cite{GKLMN19}, we will show by a coupon-collector argument that with high probability, no more than $u-R^2/\theta^2$ elements of $\{x_1,\dots,x_u\}$ are included in the sample $S$. Consider repeatedly sampling elements i.i.d. uniformly at random from $\{x_1,\dots,x_u\}$. For every $k \in \{1,\dots,u\}$, let $X_k$ be the number of samples between the time the $(k-1)$'th distinct element is sampled and the time the $k$'th distinct element is sampled. Then $X_k \sim Geom(p_k)$, where $p_k = (u-k+1)/u$. Denote $X \eqdef \sum_{k=1}^{u-t} {X_k}$ for $t=R^2/\theta^2$. Then:
\begin{equation*}
\begin{split}
\mathbb{E}[X] &= \sum_{k=1}^{u-t}{\frac{u}{u-k+1}}\\ 
&= u \left(\sum_{k=1}^{u}{\frac{1}{u-k+1}} - \sum_{k=u-t+1}^{u} \frac{1}{u-k+1} \right) \\
&= u \left(\sum_{k=1}^{u}{\frac{1}{k}} - \sum_{k=1}^{t} \frac{1}{k} \right) \\
&= u (H_{u}-H_{t}) \\
&\ge u(\ln(u) - \ln(t)-1) = u(\ln(u/t)-1).
\end{split}
\end{equation*}
For a large enough constant $C$ such that $R>C\theta$, we have $\mathbb{E}[X] \ge em$. To see why this is true, recall that $u= 4e\varepsilon^{-1} m/\ln m$, and $m \ge (R^2/\theta^2)^{1+\varepsilon}$, and therefore   
\begin{equation*}
\begin{split}
u\ln\left(\frac{u}{et}\right) &= 4e \varepsilon^{-1} \cdot \frac{m}{\ln m} \cdot \ln\left(\frac{4e \varepsilon^{-1} \cdot \frac{m}{\ln m}}{e\left(\frac{R^2}{\theta^2}\right)}\right) \\
&\ge 4e \varepsilon^{-1} \cdot \frac{m}{\ln m} \cdot \ln\left(\frac{4e \varepsilon^{-1} m^{\varepsilon/(1+\varepsilon)}}{e\ln m}\right) \\
&\ge 4e \varepsilon^{-1} \cdot \frac{\varepsilon}{2(1+\varepsilon)} \frac{m}{\ln m} \cdot \ln m \ge em\;,
\end{split}
\end{equation*}
where the inequality before last is due to the fact that for large enough $C>0$, $\ln m < m^{\varepsilon/(2(1+\varepsilon))}$.
Denote next $p_* = \min_{k \in [u-t]} p_k = (t+1)/u$, and $\lambda = m/\mathbb{E}[X]$, then $0 < \lambda \le e^{-1}$, and following known tail bounds on the sum of geometrically-distributed random variables (e.g. \cite[Theorem~3.1]{J18}) we get:
\begin{equation*}
\Pr[X \le m] \le \Pr[X \le \lambda\mathbb{E}[X]] \le e^{-p_{*}\mathbb{E}[X](\lambda - 1 - \ln \lambda)} \;.
\end{equation*}
As $\lambda \le e^{-1}$ we get that $1+\ln \lambda < 0$, and therefore 
\begin{equation*}
\Pr[X \le m] \le  e^{-\frac{t+1}{u}\cdot \mathbb{E}[X] \cdot \lambda} \le e^{-(t+1)\frac{m}{u}} \leq e^{\frac{
\varepsilon \ln m}{4e}} \;.
\end{equation*}
For large enough $C>0$ we have $e^{-(4e)^{-1}\eps \ln m} < 1/2$. 
Therefore with constant probability over $S \sim \D^m$, there are at least $t$ elements from $\{x_1,\dots,x_u\}$ that are not included in $S$. Assume we are given such an $S$. Let $x_{i_1},\dots,x_{i_{t/16}}$ denote some $t/16$ elements that are not in $S$ and consider the vector $w$ having its $(u+1)$'st coordinate set to $\theta\sqrt{2}/R$, coordinates $i_j = -2\sqrt{2} \theta/R$ and remaining coordinates $0$. Then $\|w\|_2 = \sqrt{\theta^2 2/R^2 + (t/16) 8 \theta^2/R^2} \leq \sqrt{1/8 + 1/2} < 1$. Notice that for all $x_i \in S$, we have $\langle w, x_i \rangle = (\theta \sqrt{2}/R) \cdot R/\sqrt{2} = \theta$. For an $x_{i_j}$ we have $\langle w, x_{i_j} \rangle = (\theta \sqrt{2}/R) \cdot R/\sqrt{2} + (-2\sqrt{2} \theta/R) \cdot R/\sqrt{2} = \theta - 2 \theta = -\theta$. Thus $\Loss_S^\theta(w) = 0$ while $\Loss_{\D}(w)= t/(16 u) = \Omega(R^2\ln m/(\theta^2 m))$.
\ifpdf
\subsection{Large \texorpdfstring{$\tau$}{t}}
\else
\subsection{Large $\tau$}
\fi
In this section, we prove Lemma~\ref{thm:existLargeTau}.
Let $m \geq (R^2/\theta^2)^{1+\eps}$ be the number of samples with $\eps = 0.001$, and let $R^2 \ln(m)/(\theta^2 m) < \tau \leq 1$. We construct a distribution $\D$ over $\R^{u+1} \times \{-1,+1\}$, where $u = R^2/(16 \theta^2 \tau)$. The distribution $\D$ gives a uniform random point among $\{x_1,\dots,x_u\}$ where $x_i$ has its $(u+1)$'st and $i$'th coordinate equal to $R/\sqrt{2}$ and the rest to $0$. The label is always $1$.

In our lower bound proof, we will find a vector $w$ of the following form. Let $k = e^{-28}\tau u$, and for every subset $T \subseteq \{1,\dots,u\}$ with $|T|=k$, let $w_T$ be the vector where each coordinate $i$ with $i \in T$ is set to $-1/\sqrt{2 k}$, its $(u+1)$'st coordinate is set to $\theta \sqrt{2}/R$ and all remaining coordinates are set to $0$. Then $\|w_T\|_2 = \sqrt{1/2 + 2\theta^2 /R^2} \leq 1$, as $R > C \theta$ for some sufficiently large $C>0$. In addition, for every $i \notin T$, $\langle x_i , w_T \rangle = \theta$ and for every $i \in T$ we have $\langle x_i, w_T \rangle = \theta - R/(2\sqrt{k}) \le  -\theta < 0$ if $i \in T$. Clearly for every such subset $T$, $\Loss_{\D}(w_T) = k/u = \tau/e^{28}$. What remains is to argue that with constant probability over $S$, there exists $T$ where $\Loss_S^\theta(w_T)$ is significantly smaller than $k/u$, i.e. there is a large gap between $\Loss_{\D}(w_T)$ and $\Loss_S^\theta(w_T)$.

Fix some set $S$ of $m$ samples from $\D$, let $b_i$ denote the number of times $x_i$ is in the sample. Then for every $T$ we have $\Loss_S^\theta(w_T) = (\sum_{i \in T} b_i)/m$. Let $T^*\subseteq\{1,\ldots,u\}$ be the set containing the $k$ indices with smallest $b_i$. We will show that with good probability over the choice of $S$ the $k$ smallest values among $b_1,\ldots,b_u$ are small, and thus $(\sum_{i \in T^*} b_i)/m$ is small.

Consider first a fixed index $i$. For every $j \in [m]$ let $c_j$ be the indicator for the event that the $j$'th element in the sample is $x_i$. Then $c_1,\ldots,c_m$ are independent indicators with success probability $p=1/u$, and moreover, $b_i = \sum_{j \in [m]}{c_i}$. We will use the following reverse Chernoff bound to show that $b_i$ is significantly smaller than its expectation $m/u$ with reasonable probability.
\begin{lemma}
\label{lem:reverseChernoff}[Klein and Young~\cite{reversechernoff}]
For every $\sqrt{3/(mp)} < \delta < 1/2$,
$$
\Pr\left[\sum_j c_j \leq (1-\delta)mp\right] \geq e^{-9 mp\delta^2}.
$$
\end{lemma}
Now set
$$
\delta = \sqrt{\ln(u/(2k))/(9m/u)}.
$$
Since $u/(2k) = e^{28}\tau^{-1}/2 > e^{27}$ it follows that $\delta > \sqrt{\ln(e^{27})/9(m/u)} = \sqrt{3/(m/u)}$. We have assumed $\tau > R^2 \ln(m)/(\theta^2 m)$, and thus $u = R^2/(16\theta^2\tau) <m/(16\ln m)$. Therefore $\delta = \sqrt{\ln(u/(2k))/(9m/u)} \le \sqrt{\ln(e^{28}\tau^{-1})/(9\cdot 16 \ln m)}\le 1/2$ for a large enough constant $C>0$ such that $R>C \theta$. 
Hence we may use Lemma~\ref{lem:reverseChernoff} to conclude that $\Pr[b_i \leq (1-\delta)m/u] \geq e^{-\ln(u/(2k))} = 2k/u$.

We will next show that with constant probability there are at least $k$ indices $i$ for which $b_i \leq (1-\delta)m/u$. 
Let $B_i$ denote the indicator for the event $b_i \leq (1-\delta)m/u$. We will show that with probability at least $1/8$, $B \eqdef \sum_i{B_i}\ge k$. Note first that $\E[B]=\E[\sum_i B_i] = u \E[B_1] \geq 2k$. 
By the Paley-Zygmund inequality it follows that 
\begin{equation}
\Pr\left[B \geq k\right] \ge \Pr\left[B \geq (1/2)\E\left[B\right]\right] \geq \frac{\E[B]^2}{4\E[B^2]}
\label{eq:PZ}
\end{equation}

Consider now $\E[B^2] = \sum_{i,j} \E[B_i B_j]$. For $i \neq j$, we have that the events $B_i$ and $B_j$ are negatively correlated and thus $\E[B_i B_j] \leq \E[B_i]^2 = \E[B_1]^2$. For $i=j$ we have $\E[B_i B_i] = \E[B_i] = \E[B_1]$. Therefore we may bound $\E[B^2] \leq (u^2-u)\E[B_1]^2 + u\E[B_1] \leq \E[B]^2 + \E[B]$. Note that for a large enough $C>0$, $\E[B] \geq 2k \geq 1$ and thus $\E[B] \leq \E[B]^2$ and we get that $\E[B^2] \le 2\E[B]^2$. Plugging in \eqref{eq:PZ}, we conclude that $\Pr[B \geq k] \geq 1/8$, and hence with probability at least $1/8$ over the random set of samples $S$, it holds that $(\sum_{i \in T^*} b_i)/m \leq (k(1-\delta)m/u)/m = k(1-\delta)/u$. In this case, we have $\Loss_{\D}(w_{T^*})  - \Loss_S^\theta(w_{T^*}) \geq k \delta/u = \Omega((R^2/\theta^2) \sqrt{\ln(u/k)/(m/u)})=\Omega(\sqrt{(R^2/\theta^2)\ln(\tau^{-1})\tau /m})$. Since $\tau = e^{28}k/u = e^{28}\Loss_{\D}(w_{T^*}) \geq e^{28}\Loss_S^\theta(w_{T^*})$ we have that $\Loss_S^\theta(w_{T^*}) \leq \tau/e^{28} \le \tau$ which concludes the proof of Lemma~\ref{thm:existLargeTau}.

\section{Algorithmic Lower Bound}

This section is devoted to the proof of Theorem~\ref{th:lowerAlg}. To this end, fix some integer $N$, and fix $\theta \in \left(1/N, 1/40\right)$. Let $k = \lfloor (R/\theta)^2 \rfloor$, and let $\Xs = \{Re_1,\ldots,Re_k\}$, where $e_1,\ldots,e_k$ are the standard basis elements in $\mathbb{R}^k$. Let $\A$ be a learning algorithm that, upon receiving as input a sample set $S \sim {\cal D}^m$ produces a hyperplane $w_{\A,S}$.
With every $\ell \in \{-1,1\}^u$ we associate a distribution $\D_\ell$ over $\Xs \times \{-1,1\}$ and a unit vector $w_\ell$. We show that for some labeling $\ellsp$, with constant probability over the choice of a sample $S$ of $m$ points sampled from $\D_{\ellsp}$, a large fraction of sample points attain large margins with respect to $w_{\ellsp}$, while the hyperplane $w_{A,S}$ constructed by the algorithm has a high out-of-sample error probability (with respect to $\D_{\ellsp}$). 

We first turn to define $\D_\ell$ for $\ell \in \{-1,1\}^k$. We define $\D_\ell$ separately for the first $k/2$ points and the last $k/2$ points of $\Xs$. Intuitively, every point in $\{Re_i\}_{i \in [k/2]}$ has a fixed label determined by $\ell$, however all points but one have a very small probability of being sampled according to $\D_\ell$. Every point in $\{Re_i\}_{i \in [k/2+1,k]}$, on the other hand, has an equal probability of being sampled, however its label is not fixed by $\ell$ rather than slightly biased towards $\ell$. Formally, let $\alpha, \beta, \varepsilon \in [0,1]$ be constants to be fixed later. 
For $(x,y) \sim \D_\ell$, the probability that $x \in \{Re_i\}_{i \in [k/2]}$ is $1 - \beta$. Next, conditioned on $x \in \{Re_i\}_{i \in [k/2]}$, $(Re_1,\ell_1)$ is assigned high probability $(1-\varepsilon)$ and the rest of the measure is distributed uniformly over $\{(Re_i,\ell_i)\}_{i \in [2,k/2]}$. That is
\begin{equation*}
\Pr_{\D_\ell}[(Re_1,\ell_1)] = (1-\beta)(1 - \varepsilon)\;, \;and\; \forall j \in [2,k/2].\;\; \Pr_{\D_\ell}[(Re_j,\ell_j)] = \frac{(1-\beta)\varepsilon}{k/2-1}  \;.
\end{equation*}
Finally, conditioned on $x \in \{Re_i\}_{i \in [k/2+1,k]}$, $x$ distributes uniformly over $\{Re_i\}_{i \in [k/2+1,k]}$, and conditioned on $x = Re_i$, we have $y=\ell_i$ with probability $\frac{1+\alpha}{2}$. That is
\begin{equation*}
\forall j \in [k/2+1,k]. \; \; \Pr_{\D_\ell}[(\xi_j,\ell_j)] = \frac{(1+\alpha)\beta}{k}\;, and\; \Pr_{\D_\ell}[(\xi_j,-\ell_j)] = \frac{(1-\alpha)\beta}{k}\;.
\end{equation*}
We additionally associate with $\ell$ the unit vector $w_\ell \eqdef \frac{1}{\sqrt{k}}\ell$, and draw the reader's attention to the fact that for every $i \in [k]$, $\ell_i\dotp{w_\ell}{Re_i} = (R/\sqrt{k})\ell_i^2 = \theta$. Therefore for every $(x,y) \in \supp(\D_\ell)$, we have that $y\dotp{w_\ell}{x} < \theta$ if and only if there exists $i \in [k/2+1,k]$ such that $x=Re_i$ and $y = -\ell_i$. Therefore for every $\ell \in \{-1,1\}^k$ we have 
\begin{equation}
\Pr_{(x,y)\sim\D_{\ell}}[y\dotp{w_\ell}{x} < \theta] = \sum_{i\in[k/2+1,k]}{\Pr_{(x,y)\sim\D_{\ell}}[x=Re_i \;and\; y=-\ell_i ]}=\sum_{i\in[k/2+1,k]}{\frac{(1-\alpha)\beta}{k}}=\frac{(1-\alpha)\beta}{2} \;.
\label{eq:outMargin}
\end{equation}

We will show that for some labeling $\ellsp$, with constant probability over the sample $S \sim \D_{\ellsp}^m$ and the choices of $\A$, the hyperplane $w_{\A,S}$ returned by $\A$ has a high out-of-sample error. Formally, we show the following.

\begin{claim}\label{c:ellsp}
If $\alpha \le \sqrt{\frac{k}{40\beta m}}$ and $\varepsilon \le \frac{k}{10m}$, then there exists $\ellsp \in \{-1,1\}^k$ such that with probability at least $1/11$ over $S \sim \D_{\ellsp}^m$ and the choices of $\A$ we have 
$$\Pr_{(x,y)\sim \D_{\ellsp}}[y\dotp{w_{\A,S}}{x}<0] \ge \frac{(1-\alpha)\beta}{2} + \frac{1}{12}\left((1-\beta)\varepsilon + \alpha\beta\right)\;.$$
\end{claim}

Before proving the claim, we show that it implies Theorem~\ref{th:lowerAlg}.

\begin{proof}[Proof of Theorem~\ref{th:lowerAlg}]
Fix some $\tau \in [0,49/100]$, and let $\varepsilon = \frac{u}{10m}$. Assume first that $\tau \le \frac{k}{300m}$, and let $\beta = \alpha = 0$. Then for every sample $S \sim \D_{\ellsp}^m$, $\Pr_{(x,y)\sim S}[y\dotp{w_{\ellsp}}{x} < \theta]=0 \le \tau$, and moreover by Claim~\ref{c:ellsp} with probability at least $1/11$ over $S$ and the randomness of $\A$
\begin{equation*}
\Pr_{(x,y)\sim\D_{\ellsp}}[y\dotp{w_{\A,S}}{x} < 0 ] \ge \frac{(1-\beta)\varepsilon}{12} \ge \tau + \Omega\left(\frac{k}{m} \right) = \tau + \Omega\left(\frac{R^2}{m\theta^2} + \sqrt{\frac{\tau R^2}{m\theta^2}}\right) \;.
\end{equation*}
where the last transition is due to the fact that $k = R^2\theta^{-2}$ and $\tau = O(k/m)$.

Otherwise, assume $\tau > \frac{k}{300m}$, and let $\varepsilon=0$, $\alpha = \sqrt{\frac{k}{2560\tau m}}$ and $\beta = \frac{64\tau}{32-31\alpha}$. Since $\tau \ge \frac{k}{300m}$, then $\alpha \in [0,1]$. Moreover, if $m>Ck$ for large enough but universal constant $C>0$, then $32 - 31\alpha \ge 64 \cdot \frac{49}{100} \ge 64 \tau$, and hence $\beta \in [0,1]$. Moreover, since $\alpha \le 1$ then $\beta \le 64\tau$, and therefore $\alpha = \sqrt{\frac{k}{2560 \tau m}} \le \sqrt{\frac{k}{40\beta m}}$. 
Let $\left\langle(x_1,y_1), \ldots, (x_m,y_m)\right\rangle \sim \Dy^m$ be a sample of $m$ points drawn independently according to $\D_{\ellsp}$. For every $j \in [m]$, by~\eqref{eq:outMargin} we have $\mathbb{E}[\mathbbm{1}_{y_j\dotp{w_{\ellsp}}{x_j} < \theta}] = \frac{(1-\alpha)\beta}{2}$. Therefore by Chernoff we get that for large enough $N$, 
\begin{equation*}
\begin{split}
\Pr_{S\sim \D_{\ellsp}^m}\left[ \Pr_{(x,y)\sim S}\left[y\dotp{w_{\ellsp}}{x} < \theta\right] \ge \tau\;\right] &= \Pr_{S\sim \D_{\ellsp}^m}\left[\frac{1}{m} \sum_{j \in [m]}{\mathbbm{1}_{y_j\dotp{w_{\ellsp}}{x_j} < \theta}} \ge \frac{(1-31\alpha/32)\beta}{2}\;\right] \\
&\le e^{- \Theta\left(\alpha^2\beta m \right)} \le e^{-\Theta(k)} \le 10^{-3} \;, \\
\end{split}
\end{equation*}
where the inequality before last is due to the fact that $\alpha^2 \beta m = \frac{k \beta}{2560 \tau} = \Omega(k)$, since $\beta \ge 2\tau$. Moreover, with probability at least $1/11$ over $S$ and $\A$ we get that

\begin{equation*}
\begin{split}
\Pr_{(x,y)\sim\D_{\ellsp}}[y\dotp{w_{\A,S}}{x} < 0 ] &\ge \frac{(1-\alpha)\beta}{2} + \frac{\alpha\beta}{12} = \frac{(1-31\alpha/32)\beta}{2} + \frac{\alpha\beta}{24} = \tau + \Omega\left(\sqrt{\frac{\tau k}{m}}\right)\\
&\ge \tau + \Omega\left(\frac{R^2}{m\theta^2} + \sqrt{\frac{\tau R^2}{m\theta^2}}\right) \;,
\end{split}
\end{equation*}
where the last transition is due to the fact that $\tau = \Omega(k/m)$. This completes the proof of Theorem~\ref{th:lowerAlg}.
\end{proof}

For the rest of the section we therefore prove Claim~\ref{c:ellsp}.
We first show that if $\alpha$ and $\varepsilon$ are small enough, then there exists a labeling $\ellsp$ for which the expected out-of-sample error of $w_{\A,S}$ is large. We will then use Markov's inequality to show that the out-of-sample error of $w_{\A,S}$ is large with constant probability.
More precisely, note that 
\begin{equation}
\begin{split}
&\Pr_{(x,y)\sim\D_\ell}[y\dotp{w}{x} < 0] = \\ 
&= \sum_{i \in [k/2], y \in \{-1,1\}}{\mathbbm{1}_{y\dotp{w}{Re_i}<0}\Pr_{\D_\ell}[(Re_i,y)]} + \sum_{i \in [k.2+1,k], y \in \{-1,1\}}{\mathbbm{1}_{y\dotp{w}{Re_i}<0}\Pr_{\D_\ell}[(Re_i,y)]} \;,\\
\end{split}
\label{eq:sumD}
\end{equation}
and denote $\Psi = \Psi(w, \ell) \eqdef \Pr_{(x,y)\sim\D_\ell}[y\dotp{w}{x} < 0] - \mathbbm{1}_{y\dotp{w}{Re_1}<0}\Pr_{\D_\ell}[(Re_1,y)]$. We will first lower bound the expected value of $\Psi$.
\begin{claim}
If $\alpha \le \sqrt{\frac{k}{40\beta m}}$ and $\varepsilon \le \frac{k}{10m}$, then there exists $\ellsp \in \{-1,1\}^k$ such that 
$$\mathbb{E}_{\A,S}\left[\Psi(w_{\A,S},\ellsp)] \; \right] \ge \frac{(1-\alpha)\beta}{2} + \frac{1}{6}\left((1-\beta)\varepsilon + \alpha\beta\right)\;.$$
\end{claim}

\begin{proof}
To show existence of a labeling $\ellsp$ it is enough to show that 
$$\mathbb{E}_{\ell\in\{-1,1\}^k}\left[\mathbb{E}_{\A,S}\left[\Psi(w_{\A,S},\ell) \right]\right] \ge \frac{(1-\alpha)\beta}{2} + \frac{1}{6}\left((1-\beta)\varepsilon + \alpha\beta\right)\;.$$
From \eqref{eq:sumD} we get that
\begin{equation}
\begin{split}
&\mathbb{E}_{\ell \in \{-1,1\}^k}\left[\; \mathbb{E}_{\A,S} \left[ \Psi(w_{\A,S},\ell)\right]\;\right]= \\ 
&= \mathbb{E}_{\ell}\left[\mathbb{E}_{\A,S} \left[\sum_{i \in [2,k/2], y \in \{-1,1\}}{\mathbbm{1}_{y\dotp{w_{\A,S}}{Re_i}<0}\Pr_{\D_\ell}[(Re_i,y)]} + \sum_{i \in [k/2+1,k], y \in \{-1,1\}}{\mathbbm{1}_{y\dotp{w_{\A,S}}{Re_i}<0}\Pr_{\D_\ell}[(Re_i,y)]} \right]\right] \\
\end{split}
\label{eq:lowerD}
\end{equation}
In order to lower bound the expected value of $\Psi(w_{\A,S},\ell)$ over $\ell, \A, S$, we will bound the expected value of each of the two sums in \eqref{eq:lowerD} separately, starting with the first.

For every $i \in [2,k/2]$ and $y \in \{-1,1\}$, if $y \ne \ell_i$ then $\Pr_{\D_\ell}[(Re_i,y)]=0$, and if $y = \ell_i$ then $\Pr_{\D_y}[(Re_i,y)]=\frac{(1-\beta)\varepsilon}{k/2-1}$. Therefore for every $\ell, \A, S$
\begin{equation}
\sum_{j \in [2,k/2], y \in \{-1,1\}}{\mathbbm{1}_{y\dotp{w_{\A,S}}{Re_j}<0}\Pr_{\D_y}[(Re_j,y)]} \ge \frac{(1-\beta)\varepsilon}{k/2-1}\sum_{j \in [2,k/2]}{\mathbbm{1}_{y\dotp{w_{\A,S}}{Re_j}<0}}\;.
\label{eq:lowerDFirstSum}
\end{equation}
For every $i \in [2,k/2]$, if $Re_i \notin S$ then $\A$ has no information regarding $\ell_i$, and therefore $\ell_i$ and $\dotp{w_{\A,S}}{Re_i}$ are independent. Hence $\mathbb{E}_{\ell\sim\{-1,1\}^k}[\mathbbm{1}_{\ell_i\dotp{w_{\A,S}}{Re_i}<0}]= \frac{1}{2}$. Let ${\cal S}$ be the set of all samples for which $| S \cap \{Re_2,\ldots,Re_{k/2}\}| \le \frac{k/2-1}{2}$, then for every $S \in {\cal S}$ and every set of random choices of $\A$,
\begin{equation*}
\begin{split}
\mathbb{E}_{\ell}\left[\sum_{i \in [2,k/2-1]}{\mathbbm{1}_{\ell_i\dotp{w_{\A,S}}{Re_i}<0}}\right] 
\ge \frac{k/2-1 - |S \cap \{Re_2,\ldots,Re_{k/2}\}|}{2} \ge \frac{k/2-1}{4}\;,
\end{split}
\end{equation*}
As this holds for every $S \in {\cal S}$, and every set of random choices made by $\A$ we conclude that
\begin{equation*}
\mathbb{E}_{\A,S}\left[\left.\mathbb{E}_{\ell}\left[\frac{(1-\beta)\varepsilon}{k/2-1}\sum_{j \in [2,k/2]}{\mathbbm{1}_{y\dotp{w_{\A,S}}{Re_j}<0}}\right]\; \right|S \in {\cal S}\;\right] \ge \frac{(1-\beta)\varepsilon}{k/2-1}\cdot\frac{k/2-1}{4} = \frac{(1-\beta)\varepsilon}{4}\; .
\end{equation*} 
A Chernoff bound gives $\Pr_{S \sim \D^m}[{\cal S}] \ge 1- e^{-\Theta(k)} \ge 2/3$, and by Fubini's theorem we get that
\begin{equation}
\begin{split}
&\mathbb{E}_{\ell}\left[\mathbb{E}_{\A,S}\left[\frac{(1-\beta)\varepsilon}{k/2-1}\sum_{j \in [2,k/2]}{\mathbbm{1}_{y\dotp{w_{\A,S}}{Re_j}<0}}\right]\;\right] = \mathbb{E}_{\A,S}\left[\mathbb{E}_{\ell}\left[\frac{(1-\beta)\varepsilon}{k/2-1}\sum_{j \in [2,k/2]}{\mathbbm{1}_{y\dotp{w_{\A,S}}{Re_j}<0}}\right]\;\right] \\
&\ge \mathbb{E}_{\A,S}\left[\left.\mathbb{E}_{\ell}\left[\frac{(1-\beta)\varepsilon}{k/2-1}\sum_{j \in [2,k/2]}{\mathbbm{1}_{y\dotp{w_{\A,S}}{Re_j}<0}}\right] \right| S \in {\cal S}\;\right] \cdot \Pr[S \in {\cal S}] \ge \frac{(1-\beta)\varepsilon}{6}
\end{split}
\label{eq:boundFirstSum}
\end{equation}
Next, for every $i \in [k/2+1,k]$ we have that 
\begin{equation*}
\begin{split}
\sum_{y \in \{-1,1\}}{\mathbbm{1}_{y\dotp{w}{Re_i} < 0}\Pr_{\D_\ell}[(Re_i,y)]}&=\mathbbm{1}_{\ell_i\dotp{w}{Re_i}<0}\Pr_{\D_\ell}[(Re_i,\ell_i)] + \mathbbm{1}_{\ell_i\dotp{w}{Re_i}>0}\Pr_{\D_\ell}[(Re_i,-\ell_i)] \\
&=\frac{(1-\alpha)\beta}{k} + \mathbbm{1}_{\ell_i\dotp{w}{Re_i}<0}\frac{\alpha\beta}{k/2} \;,
\end{split}
\end{equation*}
and therefore 
\begin{equation}
\sum_{i \in [k/2+1,k], y \in \{-1,1\}}{\mathbbm{1}_{y\dotp{w}{Re_i}<0}\Pr_{\D_\ell}[(Re_i,y)]} = \frac{(1-\alpha)\beta}{2} + \frac{\alpha\beta}{k/2}\sum_{i \in [k/2+1,k]}{\mathbbm{1}_{\ell_i\dotp{w}{Re_i}<0}} \;.
\label{eq:lowerDSecondSum}
\end{equation}
Next, let $i \in [k/2+1,k]$. Denote by $\sigma_i \in [m]$ the number of times $Re_i$ was sampled into $S$. Then
\begin{equation}
\mathbb{E}_{\ell}\left[\mathbb{E}_{\A,S}\left[\mathbbm{1}_{\ell_i\dotp{w_{\A,S}}{Re_i}<0}\right]\right]=\sum_{n=0}^m{\mathbb{E}_{\ell}\left[\mathbb{E}_{\A,S}\left[\left.\mathbbm{1}_{\ell_i\dotp{w_{\A,S}}{Re_i}<0}\right| \sigma_i=n\right]\right]}\cdot \Pr[\sigma_i=n]
\label{eq:totalExp}
\end{equation}
For every $a > 0$ and $b \in (0,1)$, let $\Phi(a,b) = \frac{1}{4}\left(1 - \sqrt{1-\exp\left(\frac{-ab^2}{1 - b^2}\right)}\right)$, then a result by Anthony and Bartlett \cite[Lemma~5.1]{AB09} 
shows that
\begin{equation*}
\mathbb{E}_{\ell}\left[\mathbb{E}_{\A,S}\left[\left.\mathbbm{1}_{\ell_i\dotp{w_{\A,S}}{Re_i}<0}\right| \sigma_i=n\right]\right] \ge \Phi(n+2,\alpha)
\end{equation*}
Plugging this into \eqref{eq:totalExp}, by the convexity of $\Phi(\cdot, \alpha)$ and Jensen's inequality we get that
\begin{equation*}
\mathbb{E}_{\ell}\left[\mathbb{E}_{\A,S}\left[\mathbbm{1}_{\ell_i\dotp{w_{\A,S}}{Re_i}<0}\right]\right]\ge\sum_{n=0}^m{\Phi(n+2,\alpha)}\cdot \Pr[\sigma_i=n] \ge \Phi(\mathbb{E}[\sigma_i]+2,\alpha)\;.
\end{equation*}
Since $\mathbb{E}[\sigma_i] = \frac{2\beta m}{k}$, and Since $\Phi(\cdot,\alpha)$ is monotonically decreasing we get that 
\begin{equation*}
\mathbb{E}_{\ell}\left[\mathbb{E}_{\A,S}\left[\mathbbm{1}_{\ell_i\dotp{w_{\A,S}}{Re_i}}\right]\right] \ge \Phi\left(\frac{4\beta m}{k},\alpha\right) \;.
\end{equation*}
As for $\alpha \le \sqrt{\frac{k}{40 \beta m}}$ we have $\Phi(\frac{8 \beta m}{k}, \alpha) \ge \frac{1}{6}$, summing over all $i \in [k/2+1,k]$ we get that 
\begin{equation}
\mathbb{E}_{\ell}\left[\mathbb{E}_{\A,S}\left[\sum_{i \in [k/2+1,k], y \in \{-1,1\}}{\mathbbm{1}_{y\dotp{w_{\A,S}}{Re_i}<0}\Pr_{\D_\ell}[(Re_i,y)]}\right]\right] \ge  \frac{(1-\alpha)\beta}{2}+\frac{\alpha\beta}{6}
\label{eq:boundSecondSum}
\end{equation}

Plugging \eqref{eq:boundFirstSum} and \eqref{eq:boundSecondSum} into \eqref{eq:lowerD} we conclude the claim.
\end{proof}

To finish the proof of Claim~\ref{c:ellsp}, assume $\alpha \le \sqrt{\frac{u}{40\beta m}}$ and $\varepsilon \le \frac{k}{10m}$, and let $\ellsp$ be the labeling whose existence is guaranteed by the previous claim. Note first that by substituting every indicator in \eqref{eq:sumD} with $1$, we get that $\Psi(w_{\A,S},\ellsp) \le (1-\beta)\varepsilon + \alpha \beta$ for every set of random choices made by $\A$ and every sample $S$.
Denote $a = (1-\beta)\varepsilon + \alpha \beta$. In these notations we have that $a - \Psi(w_{\A,S},\ellsp)$ is a non-negative random variable, and moreover, Claim~\ref{c:ellsp} states that $\mathbb{E}_{\A,S}[a-\Psi(w_{\A,S},\ellsp)] \le 5a/6$. Therefore from Markov's inequality we get that
\begin{equation*}
\Pr_{\A,S}[\Psi(w_{\A,S},\ellsp) \le a/12] = \Pr_{\A,S}[a - \Psi(w_{\A,S},\ellsp) \ge 11a/12] \le \Pr_{\A,S}[a - \Psi(w_{\A,S},\ellsp) \ge 1.1\mathbb{E}[a-\Psi(w_{\A,S},\ellsp)]] \le 10/11
\end{equation*}
and therefore 
$$\Pr_{\A,S}\left[\; \Pr_{(x,y)\sim\D_{\ellsp}}[y\dotp{w_{\A,S}}{x} < 0]\ge \frac{1}{12}((1-\beta)\varepsilon + \alpha \beta)\right]\ge \Pr_{\A,S}\left[\; \Psi(w_{\A,S},\ellsp) \ge \frac{1}{12}((1-\beta)\varepsilon + \alpha \beta)\right]  \ge 1/11 \;.$$

\newcommand{\etalchar}[1]{$^{#1}$}

\appendix
\section{Technical Lemmas} \label{appSec:jlDot}
This section is devoted to the proof of Lemma~\ref{l:jlDot}. We start by proving some tail bounds for norms and dot products of normal vectors.

\begin{claim}\label{c:mgfProd}
Let $X,Y \sim {\cal N}(0,1)$ be independent. Then 
\begin{enumerate}
	\item For every $\alpha < 1/2$, $\mathbb{E}\left[e^{\alpha X^2}\right] = \frac{1}{\sqrt{1-2\alpha}}$; and
	\item For every $\alpha \in (-1,1)$, $\mathbb{E}\left[e^{\alpha XY}\right] = \frac{1}{\sqrt{1-\alpha^2}}$.
\end{enumerate}
\end{claim}
\begin{proof}
To prove the first part let $\alpha < 1/2$, then
\begin{equation*}
\begin{split}
\mathbb{E}\left[e^{\alpha X^2}\right] &= \frac{1}{\sqrt{2\pi}}\int\limits_{-\infty}^{\infty}{e^{\alpha x^2}e^{\frac{-x^2}{2}}dx} = \frac{1}{\sqrt{1-2\alpha}} \cdot \sqrt{\frac{1-2\alpha}{2\pi}}\int\limits_{-\infty}^{\infty}{e^{\frac{-(1-2\alpha)x^2}{2}}dx} = \frac{1}{\sqrt{1-2\alpha}}
\end{split}
\end{equation*}

Let $\alpha \in (-1,1)$, then

\begin{equation*}
\begin{split}
\mathbb{E}\left[e^{\alpha XY}\right] &= \frac{1}{2\pi}\int\limits_{-\infty}^{\infty}{\int\limits_{-\infty}^{\infty}{e^{\alpha xy}e^{\frac{-x^2}{2}}e^{\frac{-y^2}{2}}dy}dx} = \frac{1}{2\pi}\int\limits_{-\infty}^{\infty}{\int\limits_{-\infty}^{\infty}{e^{\frac{-x^2}{2}}e^{\frac{-y^2+2\alpha xy-\alpha^2x^2}{2}}e^{\frac{\alpha^2x^2}{2}}dy}dx} \\
&= \frac{1}{2\pi}\int\limits_{-\infty}^{\infty}{e^{\frac{-x^2+\alpha^2x^2}{2}}\int\limits_{-\infty}^{\infty}{e^{\frac{-(y-\alpha x)^2}{2}}dy}dx} = \frac{1}{\sqrt{2\pi}}\int\limits_{-\infty}^{\infty}{e^{\frac{-x^2+\alpha^2x^2}{2}}dx} \\
&= \frac{1}{\sqrt{2\pi}}\int\limits_{-\infty}^{\infty}{e^{\frac{-x^2(1-\alpha^2)}{2}}dx} = \frac{1}{\sqrt{1-\alpha^2}}
\end{split}
\end{equation*}

\end{proof}

\begin{claim}\label{c:tailChiSumProd}
Let $X_1,\ldots,X_k,Y_1,\ldots,Y_k \sim {\cal N}(0,1)$ be independent, then for all $t \in [0,1/4)$, 
\begin{enumerate}
	\item $\Pr\left[\left|\frac{1}{k}\sum_{i\in [k]}{X_i^2}-1\right|\ge t\right] \le 2e^{-0.21kt^2}$; and
	\item $\Pr\left[\left|\frac{1}{k}\sum_{i\in [k]}{X_iY_i}\right|\ge t\right] \le 2e^{-0.48kt^2}$.
\end{enumerate}
\end{claim}

\begin{proof}
To prove the first part, denote $Z = \sum_{i \in [k]}{X_i^2}$ and let $t \in [0,1/4)$. For every $\alpha \in (0,1/2)$, we have that 

\begin{equation*}
\Pr[Z/k-1 > t] = \Pr[e^{\alpha Z} > e^{\alpha k (1+t)}] \le e^{-\alpha k (t+1)}\mathbb{E}[e^{\alpha Z}] = e^{-\alpha k (t+1)}(1 - 2\alpha)^{-k/2} = e^{\frac{-k}{2}\left(2\alpha (t+1) + \ln(1-2\alpha)\right)} \;.
\end{equation*}

By setting $\alpha = \frac{t}{2(t+1)}$ we get that 
\begin{equation*}
\Pr[Z/k-1 > t] \le e^{\frac{-k}{2}\left(t - \ln(1+t)\right)} \le  e^{-0.21kt^2}\;,
\end{equation*}
where the last inequality is due to the fact that for every $t \in [0,1/4)$, $\ln(1+t) \le t - 0.42t^2$.
Similarly, for $\alpha < 0$ we get that

\begin{equation*}
\Pr[Z/k-1 < -t] = \Pr[e^{\alpha Z} > e^{\alpha k (1-t)}] \le e^{-\alpha k (1-t)}\mathbb{E}[e^{\alpha Z}] = e^{-\alpha k (1-t)}(1 - 2\alpha)^{-k/2} = e^{\frac{-k}{2}\left(2\alpha (1-t) + \ln(1-2\alpha)\right)} \;.
\end{equation*}

By setting $\alpha = \frac{-t}{2(1-t)}$ we get that 
\begin{equation*}
\Pr[Z/k-1 < -t] \le e^{\frac{-k}{2}\left(-t - \ln(1-t)\right)} \le  e^{-0.25kt^2}\;,
\end{equation*}
where the last inequality is due to the fact that for every $t \in [0,1/4)$, $\ln(1-t) \le -t - 0.5t^2$.

To prove the second part of the claim, let $t \in [0,1/4)$, then for every $\alpha \in (-1,1)$, we have that 

\begin{equation*}
\Pr[Z > kt] = \Pr[e^{\alpha Z} > e^{\alpha k t}] \le e^{-\alpha k t}\mathbb{E}[e^{\alpha Z}] = e^{-\alpha k t}(1 - \alpha^2)^{-k/2} = e^{\frac{-k}{2}\left(2\alpha t + \ln(1-\alpha^2)\right)} \;.
\end{equation*}

By setting $\alpha = \frac{-1 + \sqrt{1+4t^2}}{2t}$ we get that 

\begin{equation*}
\Pr[Z > kt] \le e^{\frac{-k}{2}\left(-1 + \sqrt{1+4t^2} + \ln\left(\frac{-1+\sqrt{1+4t^2}}{2t^2}\right)\right)} = e^{\frac{-k}{2}\left(-1 + \sqrt{1+4t^2} + \ln\left(\frac{2}{1+\sqrt{1+4t^2}}\right)\right)} \;.
\end{equation*}

For every $x >0$, let $$f(x) = -1+x+\ln(2/(x+1))-0.24(x^2-1)\;.$$ Since for every $x \in [1,\sqrt{5}/2]$ we have that $f(x) \ge 0$, then for every $t \in [0,1/4)$, $f(\sqrt{1+4t^2}) \ge 0$. That is
\begin{equation*}
-1 + \sqrt{1+4t^2} + \ln\left(\frac{2}{1+\sqrt{1+4t^2}}\right) - 0.24\cdot 4t^2 \ge 0 \;.
\end{equation*}
We conclude that 

\begin{equation*}
\Pr[Z > kt] \le e^{\frac{-k}{2}\left(-1 + \sqrt{1+4t^2} + \ln\left(\frac{2}{1+\sqrt{1+4t^2}}\right)\right)} \le e^{\frac{-k}{2} \cdot 0.96t^2} = e^{-0.48kt^2} \;.
\end{equation*}

From symmetry we get that $\Pr[Z < - kt] \le e^{-0.48kt^2}$.
\end{proof}

\begin{proof}[Proof of Lemma~\ref{l:jlDot}]
The first part is follows from the standard proof of the Johnson-Lindenstrauss lemma. Every entry of $Au$ is independently ${\cal N}(0,\|u\|_2^2/k)$ distributed. 
Hence $X \eqdef \frac{k\|Au\|_2^2}{\|u\|_2^2}$ is distributed as a chi-squared 
distribution with $k$ degrees of freedom. 
From Claim~\ref{c:tailChiSumProd}, we get that $\Pr[|\|Au\|_2^2-\|u\|^2_2| \geq t\|u\|^2_2] = \Pr[|X/k-1| \geq t] \le 2e^{-0.21kt^2}$. 
\newline
To prove the second part, let $u,v \in \mathbb{R}^d$. Assume first that $\|u\|_2=\|v\|_2=1$. 
Denote $w = v - \dotp{u}{v}u$ and let $\hat{w} = w/\|w\|_2$. Note that $u \bot w$, and therefore $\|w\|_2 = \sqrt{\|v\|_2^2 - \dotp{u}{v}^2\|u\|_2^2} = \sqrt{1-\dotp{u}{v}^2}$.
For every $i \in [k]$, let $a_i$ be the $i$th row of $A$ and let $X_i \eqdef \dotp{a_i}{u}$ and $Y_i \eqdef \dotp{a_i}{\hat{w}}$. By the rotational invariance of Gaussians and orthonormality of $u$ and $\hat{w}$ we get that $X_1,\ldots,X_k,Y_1,\ldots,Y_k \sim {\cal N}(0,1/k)$ are independent. Next, observe that 

\begin{equation*}
\dotp{Au}{Av} = \dotp{Au}{A(\dotp{u}{v}u)} + \dotp{Au}{A(v-\dotp{u}{v}u)} = \dotp{u}{v} \|Au\|_2^2 + \|w\|_2\dotp{Au}{A\hat{w}} \;,
\end{equation*}

and moreover, $\dotp{Au}{A\hat{w}} = \sum_{i \in [k]}{X_iY_i}$.
Therefore 

\begin{equation}
|\dotp{Au}{Av} - \dotp{u}{v}| \le |\dotp{u}{v}|\cdot| \|Au\|_2^2 - 1| + \|w\|_2\cdot|\dotp{Au}{A\hat{w}}|
\end{equation}

Next, let $t \in [0,1/4)$, and let $\alpha \in [0,1]$ then

\begin{equation}
\begin{split}
\Pr&\left[|\dotp{Au}{Av} - \dotp{u}{v}|>t\right] \\
&\le \Pr\left[|\dotp{u}{v}|\cdot| \|Au\|_2^2 - 1| > \alpha t\right] + \Pr\left[\|w\|_2\cdot \sum_{i \in [k]}{X_iY_i} \ge (1-\alpha)t\right]
\end{split}
\label{eq:triangleIneq}
\end{equation}

From the first part of the lemma we get that 
$$\Pr\left[|\dotp{u}{v}|\cdot| \|Au\|_2^2 - 1| > \alpha t\right] \le 2e^{\frac{-0.21k\alpha^2t^2}{\dotp{u}{v}^2}} \;,$$
and from Claim~\ref{c:tailChiSumProd} we get that 

$$\Pr\left[\|w\|_2\cdot \sum_{i \in [k]}{X_iY_i} \ge (1-\alpha)t\right] \le 2e^{\frac{-0.48k(1-\alpha)^2t^2}{\|w\|_2^2}}=2e^{\frac{-0.48k(1-\alpha)^2t^2}{1 - \dotp{u}{v}^2}}\;.$$

Setting $\alpha = \frac{\sqrt{0.48}\dotp{u}{v}}{\sqrt{0.48}\dotp{u}{v}+ \sqrt{0.21(1 - \dotp{u}{v}^2})}$ and plugging into \eqref{eq:triangleIneq} we get that

\begin{equation*}
\Pr\left[|\dotp{Au}{Av} - \dotp{u}{v}|>t\right] \le 4e^{\frac{-0.48 \cdot 0.21 kt^2}{(\sqrt{0.48}\dotp{u}{v} + \sqrt{0.21(1 -\dotp{u}{v}^2)})^2}} \le 4e^{\frac{-0.48 \cdot 0.21 kt^2}{0.69}} = 4e^{-kt^2/7} \;,
\end{equation*}

where the inequality before last is due to the fact that $\sqrt{0.48}x + \sqrt{0.21(1-x^2)} \le \sqrt{0.69}$ for all $x \in [-1,1]$.
Finally, for general $u,v \in \mathbb{R}^d$ we get that since $u'= u/\|u\|_2$
and $v'= v/\|v\|_2$ are unit vectors then

\begin{equation*}
\Pr_A[|\dotp{Au}{Av} - \dotp{u}{v}| > t] = \Pr_A\left[\left|\dotp{Au'}{Av'} - \dotp{u'}{v'}\right| > \frac{t}{\|u\|_2\|v\|_2}\right] \le 4e^{-\frac{kt^2}{7\|u\|_2^2\|v\|_2^2}}
\end{equation*}

\end{proof}

\end{document}